\documentclass[sigconf]{acmart}

\usepackage[utf8]{inputenc} 
\usepackage[T1]{fontenc}    
\usepackage{hyperref}       
\usepackage{url}            
\usepackage{booktabs}       
\usepackage{amsfonts}       
\usepackage{nicefrac}       
\usepackage{microtype}      

\usepackage{todonotes}
\usepackage{amsmath} 
\usepackage{bbold} 
\usepackage{euscript}
\usepackage{epsfig} 
\usepackage{graphicx}
\usepackage{listings} 
\usepackage{moreverb} 
\usepackage{color}
\usepackage{xspace}
\usepackage{makecell}
\usepackage{tablefootnote}
\usepackage{amsthm}
\usepackage{natbib}
\newtheorem{definition}{Definition}
\newtheorem*{remark}{Remark}

\newcommand{\setone}{\mathcal{A}}
\newcommand{\settwo}{\mathcal{B}}

\newtheorem{theorem}{Theorem}
\usepackage{bbm}

\AtBeginDocument{%
  \providecommand\BibTeX{{%
    \normalfont B\kern-0.5em{\scshape i\kern-0.25em b}\kern-0.8em\TeX}}}

\setcopyright{rightsretained}

\begin{document}

\fancyhead{}

\title{Practical Compositional Fairness: Understanding Fairness in Multi-Component Recommender Systems}

\author{Xuezhi Wang, Nithum Thain, Anu Sinha, Flavien Prost, Ed H. Chi, Jilin Chen, Alex Beutel}
\affiliation{Google Brain}
\email{{xuezhiw, nthain, aradhanas, fprost, edchi, jilinc, alexbeutel}@google.com}

\renewcommand{\shortauthors}{Wang et al.}

\begin{abstract}
  How can we build recommender systems to take into account fairness?  Real-world recommender systems are often composed of multiple models, built by multiple teams.  
However, most research on fairness focuses on improving fairness in a single model.  Further, recent research on classification fairness has shown that combining multiple ``fair'' classifiers can still result in an ``unfair'' classification system.  This presents a significant challenge: how do we understand and improve fairness in recommender systems composed of multiple components?
  
  In this paper, we study the compositionality of recommender fairness.
  We consider two recently proposed fairness ranking metrics: equality of exposure and pairwise ranking accuracy.
  While we show that fairness in recommendation \emph{is not} guaranteed to compose, we provide theory for a set of conditions under which fairness of individual models \emph{does} compose.
  We then present an analytical framework for both understanding whether a real system's signals can achieve compositional fairness, and improving which component would have the greatest impact on the fairness of the overall system.
  In addition to the theoretical results, we find on multiple datasets---including a large-scale real-world recommender system---that the overall system's end-to-end fairness is largely achievable by improving fairness in individual components.  
\end{abstract}



\copyrightyear{2021}
\acmYear{2021}

\acmConference[WSDM '21]{Proceedings of the Fourteenth ACM International Conference on Web Search and Data Mining}{March 8--12, 2021}{Virtual Event, Israel}
\acmBooktitle{Proceedings of the Fourteenth ACM International Conference on Web Search and Data Mining (WSDM '21), March 8--12, 2021, Virtual Event, Israel}\acmDOI{10.1145/3437963.3441732}
\acmISBN{978-1-4503-8297-7/21/03}

\maketitle

\section{Introduction}
As recommender systems become more central in our lives, the importance of their fairness has become increasingly clear.
Over the past few years, the research community has developed a variety of metrics for evaluating the impact of recommender systems on different groups of stakeholders, such as if relevant items get equal exposure, if demographic groups of item producers rank well conditioned on user interest, and if items from a diverse set of item producers or topics are shown \cite{DBLP:conf/kdd/SinghJ18, biega2018equity, DBLP:conf/kdd/BeutelCDQWWHZHC19, diaz2020evaluating,zehlike2017fa,celis2018ranking,burke2017multisided,yao2017beyond}.  These are important values to consider in the \textit{responsible} design of a recommender system, but how to build recommender systems to meet these goals is still difficult.

One fundamental challenge is that real-world recommendation products are often composed of many models, designed to capture different aspects of the user experience and sometimes even built and maintained by distinct teams \cite{survey05, hybrid02,he2014practical, wang2011cascade}.  For example, one team may be responsible for predicting implicit feedback signals like clicks \cite{he2014practical}, another may be responsible for predicting explicit feedback like ratings or surveys \cite{bennett2007netflix,barbieri2016improving,yi2014beyond,garcia2018understanding}, and another may be responsible for predicting notions of item quality \cite{potthast2016clickbait,kumar2018false}; these different predictions are typically then combined, often multiplicatively, to produce a final ranking \cite{zhe2019, xiao2018,youtube2010,claypool1999combing}.  

This multi-component design pattern poses a challenge for recommendation fairness: how should each team build their models to make the end-ranking seen by the user meet fairness goals?  Nearly all of the research on training-time improvements for fairness has assumed that the ``system'' that we are improving is a single differentiable model \cite{DBLP:conf/icml/AgarwalBD0W18,beutel2019putting, DBLP:conf/kdd/BeutelCDQWWHZHC19,singh2019policy,narasimhan2019pairwise,pmlr-v54-zafar17a}\footnote{Note, we focus on training-time improvements because serving-time changes typically require knowing demographics at serving-time, which is often not possible \cite{beutel2019putting}.}.  As a result, an obvious guess is to train each model in the recommender system to itself meet a ranking fairness goal.  Unfortunately, recent research on fairness in classification showed that even if two classifiers are ``fair,'' a combination of their predictions can still be ``unfair'' \cite{dwork2018fairness,cynthia2018}. 
Known as the compositional fairness problem, this presents a significant open question for recommenders: does recommendation fairness compose in our modular recommender systems?

In this paper, we study the compositional recommender fairness problem both from a worst-case lens and a data-driven average-case lens.  
First, we demonstrate that recommendation fairness, unfortunately, \emph{is not guaranteed} to compose, which suggests a significant obstacle for building  real-world recommender systems for fairness.  

\vspace{-0.05in}
\paragraph{A motivating example.} 
To make the problem concrete, consider the hypothetical example of a recommendation system for books (e.g., \cite{fairness_book18}) that wants to rank items by expected user satisfaction.
It can be built with two components: (1) $pCTR = P({\rm click})$: one that predicts click-through rate on a book; (2) $pRating = E[{\rm rating} | {\rm click}]$: one that predicts the star rating given a book was clicked, akin to \cite{xiao2018}. With these components, items can then be ranked by:
\vspace{-0.03in}
$$
E[{\rm rating}] = P({\rm click}) \times E[{\rm rating} | {\rm click}] = pCTR \times pRating.
$$
Let the fairness goal be demographic parity in \textit{ranking exposure} \cite{DBLP:conf/kdd/SinghJ18}. 
That is, the ranking of the composite score
should not systematically differ between \textit{white} and \textit{non-white} authors.
Each component could be made ``fair'' with respect to author demographics through recent mitigation methods \cite{DBLP:conf/kdd/BeutelCDQWWHZHC19,singh2019policy,narasimhan2019pairwise}. What does this mean for the demographic parity on the ranked composite scores? It may feel intuitive to assume that if each component gives equal exposure to each group, the overall system should as well.
Here we show through an example that this is not the case. Assume we have $4$ books with author demographics in the header of each column:
\begin{center}
\begin{tabular}{c|c|c|c|c} 
 Component & non-white & non-white & white & white \\ 
 \hline
 $pCTR$ & $0.1$ & $0.4$ & $0.2$ & $0.3$ \\ 
 $pRating$ & $0.4$ & $0.1$ & $0.3$ & $0.2$ \\  
 \hline
  $pCTR \times pRating$  & $0.04$ & $0.04$ & $0.06$ & $0.06$\\ 
\end{tabular}
\end{center}
Each component exposes books from each group equally: if we rank the items by either $pCTR$ or $pRating$ individually, we get [non-white, white, white, non-white].
However, when the two components are multiplied together, the composite ranking ([white, white, non-white, non-white]) does not result in demographic parity\footnote{One might ask: why not just rank by \emph{either} pCTR or pRating? Unfortunately, doing that would result in either increased clickbait (if only focusing on pCTR) or unappealing-looking items (if ignoring pCTR).  This is why recommender systems incorporate multiple signals and limits the flexibility of the composition function.}. 
In Section \ref{sec:definition}, we provide a broader theoretical analysis. 
We study two recently proposed ranking fairness metrics \cite{DBLP:conf/kdd/SinghJ18, biega2018equity, DBLP:conf/kdd/BeutelCDQWWHZHC19, diaz2020evaluating}, each capturing slightly different goals.
For both, we consider the broader class of multiplicative composition, as it is common in production recommender systems \cite{zhe2019, xiao2018}, and without loss of generality extends to additive composition by moving to the log-space, as is also common in practice \cite{youtube2010,claypool1999combing}.

In response to these bleak results, 
we observe that these are worst-case guarantees, and thus flip to the constructive question: \emph{when does} recommender fairness compose?  What can we do in a real recommender system?  We explore this both theoretically and empirically.  We first prove a set of conditions about the models' predictions under which recommender fairness would compose across components. With this theoretical understanding, we develop an analytical framework for counterfactually testing a multi-model recommender system.  That is, given an existing recommender system, we provide a method for measuring how well the system would meet a given recommender fairness goal if each model in the system was (independently) made fair under that metric.  We find across multiple applications, including a real-world recommender, that these fairness goals can largely be met through this per-component approach, both through our analytical framework and then through modeling experiments.  Further, we demonstrate that counterfactual testing enabled by the analytical framework can be used to uncover which components are most limiting fairness metrics, which could then be used to prioritize modeling improvements.  Taken together, we believe this suggests a path toward improving fairness in recommender systems.
In summary, our contributions in this paper are:

\textbf{Theory:} We provide theory showing a set of conditions under which fair components can compose into fair systems.

\textbf{System Understanding:} We provide an analytical framework of counterfactual testing to measure whether a system's signals can achieve compositional fairness and to diagnose which of these signals lowers the overall system fairness most.

\textbf{Experiments:} Although compositional fairness is theoretically not guaranteed, on multiple data-sets, including a large-scale production recommender system, we apply our counterfactual testing and find that the overall system fairness is largely achievable by improving fairness in individual components. We subsequently test a training-time regularization and find that while no individual component can be improved to make the system ``fair,'' applying the approach to all of the components (separately) is highly effective. 

\section{Related Work}

\textbf{Fairness in Classification:}
The majority of the fairness metric definition literature focuses on classification. 
\textit{Demographic parity} \cite{Calders:2009:BCI:1674660.1677211, liobait2015OnTR, Zafar2015LearningFC} is a common way of addressing discrimination against protected attributes. It requires a decision to be independent of the protected attribute. \textit{Equalized odds} \cite{hardt2016equality} require a predictor $\hat{Y}$ to be independent of the protected attribute, conditioned on the true label $Y$. 
Metrics have also been explored over continuous scores from a classifier: \cite{kallus2019fairness,borkan2019nuanced,narasimhan2019pairwise} all break down AUC of these scores into Mann-Whitney U-tests \cite{mwu_test}.

\textbf{Fairness in Ranking:}
Recently, a few definitions have been proposed for fairness in the \textit{ranking} setting \cite{zehlike2017fa}; in our work we focus on two recent framings.
\citet{DBLP:conf/kdd/SinghJ18} propose measuring the exposure an item or group of items gets depending on what position they fall in a ranking. (Similar metrics were proposed by \citet{biega2018equity}, and further studied by \citet{diaz2020evaluating}.)  The work offers multiple fairness goals, such as exposure proportional to relevance, but in our usage we build on this notion of exposure to measure group representation throughout a ranked list (as we ignore any label or relevance, this is philosophically closer to the principles of demographic parity above).
\citet{DBLP:conf/kdd/BeutelCDQWWHZHC19} focus on measuring accuracy in a recommender system based on \textit{pairwise comparisons}. 
The accuracy of a ranking for a pair of items is defined as the probability that the clicked item is ranked higher than the un-clicked item. 
In this set-up, two items from different groups are used to create a pair, and the difference in accuracy for each group is used as a fairness metric.  We use this metric to capture fairness in ranking more closely aligned with equal opportunity (as it is measuring accuracy with respect to a label---clicks).
In Section~\ref{sec:definition}, we will formalize the above two definitions within the same framework.
For each of the ranking metrics, given per-component fairness, we will show conditions where compositional fairness holds (and counter-examples where it might not hold).
 
A closely related area to ranking fairness is ranking diversification \cite{diverse_icml10, diverse_icml08, diverse_www09, diverse_vldb11, diverse_wsdm09, diverse_sigir98}, where the goal is to diversify the ranking results to improve user satisfaction.
In many cases, general-purpose diversification does not align with fairness for certain sub-groups.

\textbf{Compositional Fairness:} 
\citet{dwork2018fairness} studied general constructions for fair composition, and showed that classifiers that are fair in isolation do not necessarily compose into fair systems, for individuals or for groups. 
The authors studied the ``Functional Composition" setting, where the assumption is that the binary outputs of multiple classifiers are combined through \textit{logical operations} to produce a single output for a single task.
More recent work by \citet{dwork2020individual} explores composition of individual fairness in pipelines.
In this paper, we focus on fairness in recommender systems and a general multiplicative compositional setting (or equivalently additive composition in the log-space), as is common in many real ranking systems \cite{zhe2019, xiao2018, youtube2010}.

\textbf{Mitigation:}
Many mitigation approaches try to achieve fairness in the \textit{single-task} (non-compositional) setting.
The approaches can be partitioned by when they intervene in the creation of an ML system: pre-processing (e.g., \cite{tolga2016}), post-processing (e.g., \cite{calibration17,post19, DBLP:conf/kdd/SinghJ18,DBLP:conf/kdd/GeyikAK19}),
and training, including adding constraints \cite{gupta16, gupta19,DBLP:conf/icml/AgarwalBD0W18}, regularization \cite{pmlr-v54-zafar17a,beutel2019putting}, and adversarial learning \cite{variation16,adv_learning18,beutel17,DBLP:conf/icml/MadrasCPZ18}.
Post-processing approaches do not suffer from compositional challenges, but are often not feasible in practice due to not having demographic information at inference time.
The regularization approaches are most similar to our analysis, encouraging matching the distribution of predictions \cite{pmlr-v54-zafar17a,beutel2019putting} or representation \cite{beutel17,DBLP:conf/icml/MadrasCPZ18} across groups.

\section{Metrics and Theoretical Analysis}
\label{sec:definition}
In this section, we formally define the problem:
denote $x$ as the item being ranked, and there are two groups of users being considered, Group $\mathcal{A}$ and Group $\mathcal{B}$. 
Assume a recommender system with $K$ components, where the $k$-th component takes an input $x$ and produces its own score $f_k(x)$;
the broader class of multiplicative composition (\cite{zhe2019, xiao2018}) can be formulated as: $f(x)=\prod_{k=0}^{K-1} f_k(x)$.
If all $K$ components satisfy some fairness metric by themselves $F[f_k]$, 
we ask whether the overall function $f(x)$ satisfies the same fairness metric $F[f]$. Restated, we would like to know whether the system achieves compositional (end-to-end) fairness given that each component has achieved fairness independently.

In the sections below, we consider two commonly-used fairness metrics in recommender systems: \textit{ranking exposure} \cite{DBLP:conf/kdd/SinghJ18}, and \textit{pairwise ranking accuracy} \cite{DBLP:conf/kdd/BeutelCDQWWHZHC19}. We describe each metric, and explore how the function composition affects end-to-end fairness.

\subsection{Fairness in Ranking Exposure}

\subsubsection{Definition.}
\label{sec:def_exposure}
We start with the ranking exposure \cite{DBLP:conf/kdd/SinghJ18} fairness metric.
Formally, the ranking exposure for Group $\mathcal{A}$ is defined as:
\[
\text{Exposure}(\mathcal{A}|r) = \sum\nolimits_{x\in \mathcal{A}} u(x|r),
\]
under a ranking order $r$. Here $u$ denotes the utility function, which is usually a monotonically decreasing function with respect to the rank of an item. One common choice is to use an exponent $w>=0$, and $u(x|r) = [\text{rank}(x|r)]^{-w}$.
The fairness exposure metric between Group $\mathcal{A}$ and Group $\mathcal{B}$ under a ranking order $r$ is defined as:
\begin{equation}
\text{Gap}_{\text{exp}}(\mathcal{A}, \mathcal{B}|r) = \frac{|\text{Exposure}(\mathcal{A}|r) - \text{Exposure}(\mathcal{B}|r)|}{\text{Exposure}(\mathcal{A}|r) + \text{Exposure}(\mathcal{B}|r)},
\label{eq:exposure_gap}
\end{equation}
which denotes the normalized gap between the exposure for Group $\mathcal{A}$ and $\mathcal{B}$.
Note, when the two groups have the same size, i.e., $|\mathcal{A}|=|\mathcal{B}|$, the ideal exposure gap should reach zero in order to reach demographic parity \cite{Calders:2009:BCI:1674660.1677211, liobait2015OnTR, Zafar2015LearningFC}. This might not be the case when the two groups have different sizes.

Intuitively, this metric (here unconditioned on relevance) makes the goal to provide a diverse ranking with each group being well represented throughout the ranked list.  While we build on \cite{DBLP:conf/kdd/SinghJ18} for framing, similar intuitions were also proposed in \cite{zehlike2017fa,biega2018equity} and used in job search applications \cite{DBLP:conf/kdd/GeyikAK19}.

\subsubsection{Counter-example of Composition.}
We offer a counter-example and show that under the fair exposure metric, per-component fairness does not always guarantee compositional fairness.
Consider a ranking system composed of two components $f_0, f_1$, two groups $\mathcal{A}, \mathcal{B}$, and each group has two items.
For any $a>0, \epsilon>0$, suppose: 
\begin{center}
\setlength\tabcolsep{3pt}
\begin{tabular}{c|c|c|c|c} 
 Component & $x^a_1\in \mathcal{A}$ & $x^a_2\in \mathcal{A}$ & $x^b_1\in\mathcal{B}$ & $x^b_2\in \mathcal{B}$  \\ 
 \hline
 $f_0(x)$ & $a+\epsilon$ & $a+4\epsilon$ & $a+2\epsilon$ & $a+3\epsilon$ \\ 
 $f_1(x)$ & $a+4\epsilon$ & $a+\epsilon$ & $a+3\epsilon$ &$a+2\epsilon$\\
 \hline
 Exposure of $f_0(x)$ & $e_1$&  $e_0$ & $e_1$ &  $e_0$\\
 Exposure of $f_1(x)$ & $e_0$&  $e_1$ & $e_0$ &  $e_1$\\
 \hline
 Exposure of $f_0(x)\cdot f_1(x)$ & $e_1$&  $e_1$ & $e_0$ &  $e_0$
\end{tabular}
\end{center}

Here we assume the first two positions receive a higher exposure $e_0$ than the exposure $e_1$ for the last two positions (i.e., $1\geq e_0 > e_1 \geq 0$), and it is easy to see that each component by itself is fair, since $\text{Exposure}(\mathcal{A}|r_k) = \text{Exposure}(\mathcal{B}|r_k)=e_0+e_1,\ k \in \{0, 1\}$, 
$r_k$ being the ranking order determined by $f_k(x)$. 
But when combined by $f(x) = f_0(x) \cdot f_1(x)$, Group $\mathcal{A}$ is always ranked below Group $\mathcal{B}$. Specifically, $\text{Gap}_{\text{exp}}(\mathcal{A}, \mathcal{B}|r) = \frac{|e_0-e_1|}{e_0+e_1}$, and could be as large as $1$ if $e_0=1, e_1=0$.
We can also see that the magnitude of the scores does not play a role here, since we can make $\epsilon$ arbitrarily small to make the magnitude of the score for the two components arbitrarily close to each other ($\rightarrow a$), while keeping $\text{Gap}_{\text{exp}}(\mathcal{A}, \mathcal{B}|r) = 1$.

\subsubsection{Condition for composition of ranking exposure}
Now we present theory showing under what conditions we will achieve compositional fairness given we have per-component fairness, under the ranking exposure (\S\ref{sec:def_exposure}) metric. 

Consider a system with two components $f_0$ and $f_1$, and two groups $\mathcal{A}, \mathcal{B}$.
Let $X^0_A$ represent the random variable defined by $\log f_0(x), x\in \mathcal{A}$, and $X^0_B$ for $\log f_0(x), x\in \mathcal{B}$. Similarly we define $X^1_A, X^1_B$ for $f_1$.
Assume the top half and the bottom half of the items in the ranked list receive different exposure values $u(x|r)$ by sorting $f_0$ and $f_1$ in descending order, respectively.
Assume we have per-component fairness in the log-space, i.e., $\mathrm{median}[X^0_A] = \mathrm{median}[X^0_B]$ for $\log f_0(x)$ and 
$\mathrm{median}[X^1_A] = \mathrm{median}[X^1_B]$ for $\log f_1(x)$.
We have the following theorem:
\begin{theorem}
If $X^0_A, X^0_B, X^1_A, X^1_B$ are symmetric random variables such that $X^0_A + X^1_A$ and $X^0_B+X^1_B$ are also symmetric, then per-component fairness on $\log f_0$ and $\log f_1$ means we have compositional fairness for $\log f(x)$, where $f(x) = f_0(x) \cdot f_1(x)$.
\label{thm:exposure}
\end{theorem}
\begin{proof}
By composing $f_0$ and $f_1$ we have:
${\mathrm{median}_{x\in \mathcal{A}}}[\log(f_0(x) \cdot f_1(x))]
= {\mathrm{median}}_{x\in \mathcal{A}}[\log f_0(x) + \log f_1(x)] 
= \mathrm{median}[X^0_A + X^1_A]$, and since:
\begin{align*}
&\mathrm{median}[X^0_A + X^1_A] = \mathrm{mean}[X^0_A + X^1_A] 
\ \ \text{(by symmetry of $X^0_A + X^1_A$)}\\
=& \mathrm{mean}[X^0_A] + \mathrm{mean}[X^1_A] 
\qquad\qquad\qquad\text{(by linearity of expectation)} \\
=& \mathrm{median}[X^0_A] + \mathrm{median}[X^1_A]
\qquad\qquad\ \ \text{(by symmetry of $X^0_A$ and $X^1_A$)}\\
=& \mathrm{median}[X^0_B] + \mathrm{median}[X^1_B] 
\qquad\qquad\ \ \text{(by per-component fairness)}\\
=& \mathrm{mean}[X^0_B] + \mathrm{mean}[X^1_B] 
\qquad\qquad\qquad\text{(by symmetry of $X^0_B$ and $X^1_B$)}\\
=& \mathrm{mean}[X^0_B + X^1_B]
\qquad\qquad\qquad\qquad\quad\text{(by linearity of expectation)}\\
=& \mathrm{median}[X^0_B + X^1_B],
\qquad\qquad\qquad\quad\ \ \ \ \text{(by symmetry of $X^0_B + X^1_B$)}
\end{align*}
which equals to ${\mathrm{median}}_{x\in \mathcal{B}}[\log f_0(x) + \log f_1(x)]$ and thus \\${\mathrm{median}}_{x\in \mathcal{B}}[\log(f_0(x) \cdot f_1(x))]$.
\end{proof}
To empirically test whether a random variable $X$ is symmetric around a given mean $\theta$, one could construct samples $(X, 2\theta-X)$ and test whether the two distributions are the same by a two-sample Kolmogorov–Smirnov test or a kernel two-sample test \cite{mmd}. 
Further, for compositional fairness in the entire system, we need to convert $\log f_0, \log f_1, \log f$ back to the original space $f_0, f_1, f$. The median element remains unchanged by the monotonicity of the $\log$ function when there is an odd number of samples in each group, in which case we will have 
$
{\mathrm{median}}_{x\in \mathcal{A}}[f_0(x) \cdot f_1(x)]=
{\mathrm{median}}_{x\in \mathcal{B}}[f_0(x) \cdot f_1(x)].
$
When there is an even number of samples per group, then we also need the median in the original space to be the same across groups in order for compositional fairness to hold.

\subsection{Fairness as Pairwise Ranking Accuracy}
\label{sec:def_pairwise}
\subsubsection{Definition.}
Another commonly used fairness metric in ranking is \textit{pairwise ranking accuracy} \cite{DBLP:conf/kdd/BeutelCDQWWHZHC19}, where the idea is to compute the accuracy of a system ranking a pair of items correctly conditioned on the true outcome. The pair of items is constrained to come from two different groups, $\mathcal{A}$ and $\mathcal{B}$, through randomized experiments.
Formally, the metric is defined as:
\begin{align*}
\text{PairAcc}(\mathcal{A} > \mathcal{B}|r) 
&=\ P(f(x_i) > f(x_j) | y_i > y_j, x_i\in \mathcal{A}, x_j\in \mathcal{B}) \\
&=\ \frac{P(f(x_i) > f(x_j), y_i > y_j|x_i\in \mathcal{A}, x_j\in \mathcal{B})}{P(y_i > y_j|x_i\in \mathcal{A}, x_j\in \mathcal{B})}.
\end{align*}
Here $y_i\in\{0, 1\}$ denotes the observed binary outcome for $x_i$ (e.g., $y_i=1$ means a recommended item is clicked, or a recommended product is purchased).
Correspondingly, the Pairwise Ranking Accuracy Gap is defined as:
\begin{align}
\text{Gap}_{\text{pair}}(\mathcal{A}, \mathcal{B}|r)= |\text{PairAcc}(\mathcal{A} > \mathcal{B}|r) - \text{PairAcc}(\mathcal{A} < \mathcal{B}|r)| \label{eq:pairwise_gap}
\end{align}
\begin{remark}
Intuitively, this metric means that given a pair of items, one from group $\mathcal{A}$ and one from group $\mathcal{B}$, conditioned on one with $y_i=1$ and the other with $y_i=0$,
we would like the system to have the same accuracy of ranking this pair of items correctly, regardless of which group the item with positive outcome ($y_i=1$) is from.
\end{remark}

Further, let $X^k_{A_0}$ represent the random variable defined by $f_k(x_i)$ for $y_i=0$ and $x_i\in \mathcal{A}$; and let $X^k_{A_1}$ represent the random variable defined by $f_k(x_i)$ for $y_i=1$ and $x_i\in \mathcal{A}$. $X_{B_0}, X_{B_1}$ are defined similarly. 
We can simply write the Pairwise Ranking Accuracy Gap metric as:
$
\text{Gap}_{\text{pair}}(\mathcal{A}, \mathcal{B}|r) = |P(X_{A_1} > X_{B_0}) - P(X_{B_1} > X_{A_0})|.
$

\subsubsection{Counter-example of composition.} In the following we present a simple example that shows per-component fairness might not lead to compositional fairness, under the pairwise ranking accuracy gap metric.
Consider the following system with two components, two groups, and two corresponding items within each group:
\begin{center}
\begin{tabular}{c|cc|cc}
Component & $X^k_{A_1}$ & $X^k_{B_0}$ & $X^k_{B_1}$ & $X^k_{A_0}$\\
\hline
$f_0(x) $ & $\{1, 4\}$ & $\{2, 3\}$ & $\{1, 4\}$ & $\{2, 3\}$\\ 
$f_1(x) $ & $\{4, 1\}$ & $\{3, 2\}$ & $\{1, 4\}$ & $\{3, 2\}$\\
\hline
PairAcc of $f_0(x)$ & \multicolumn{2}{c|}{0.5} & \multicolumn{2}{c}{0.5}\\
PairAcc of $f_1(x)$ & \multicolumn{2}{c|}{0.5} & \multicolumn{2}{c}{0.5}\\
\hline
$f_0(x) \cdot f_1(x)$ & $\{4, 4\}$ & $\{6, 6\}$ & $\{1, 16\}$ & $\{6, 6\}$\\ 
\hline
PairAcc of $f_0(x) \cdot f_1(x)$ & \multicolumn{2}{c|}{0.0} & \multicolumn{2}{c}{0.5}\\
\end{tabular}
\end{center}
It is easy to see that each component is fair.
When composed, we have $\text{PairAcc}(\mathcal{A} > \mathcal{B}|r)=0.0$, because both items with $y_i=1$ from $\mathcal{A}$ receive a lower prediction score $\{4, 4\}$ than  items with $y_i=0$ from $\mathcal{B}$: $\{6, 6\}$.
Similarly we have $\text{PairAcc}(\mathcal{A} < \mathcal{B}|r)=0.5$ and hence $\text{Gap}_{\text{pair}}(\mathcal{A}, \mathcal{B}|r)=0.5$.
In other words, the predictor $f$ does not have equal treatment for ranking the items from $\mathcal{A}$ and $\mathcal{B}$.

\subsubsection{Condition for composition of pairwise ranking accuracy}
We provide a theorem showing under which conditions compositional fairness holds under the pairwise ranking accuracy gap metric (Eq.~\eqref{eq:pairwise_gap}). 
We denote $\setone_0$ and $\setone_1$ to be the set of items from $\setone$ with $y=0$ and $y=1$, respectively; we similarly define $\settwo_0$ and $\settwo_1$.
We assume equal size on the groups, i.e., $|\setone_\ell|=|\settwo_\ell|, \ell=\{0, 1\}$.\footnote{Note, in order to form pairs, the definition of pairwise ranking accuracy additionally implies $|\mathcal{A}_1|=|\mathcal{B}_0|$ and  $|\mathcal{A}_0|=|\mathcal{B}_1|$.} 
We define a delta term between all pairs from $\mathcal{A}_1$ and $\mathcal{B}_0$ and similarly between all pairs from $\mathcal{B}_1$ and $\mathcal{A}_0$, i.e.,
\begin{align}
\Delta_k(\mathcal{A}_1, \mathcal{B}_0)=\{f_k(x_i) - f_k(x_j) \mid x_i\in \mathcal{A}_1, x_j\in \mathcal{B}_0\};\nonumber\\
\Delta_k(\mathcal{B}_1, \mathcal{A}_0)=\{f_k(x_i) - f_k(x_j) \mid x_i\in \mathcal{B}_1, x_j\in \mathcal{A}_0\}.
\label{eq:delta_pairwise_ranking_acc}
\end{align}

Let $Z^0_{A_1, B_0}, Z^0_{A_0, B_1}$ represent the random variables defined by $\Delta_0(\mathcal{A}_1, \mathcal{B}_0)$, $\Delta_0(\mathcal{B}_1, \mathcal{A}_0)$, respectively, for $f_0$. Similarly we define $Z^1_{A_1, B_0}$, $Z^1_{A_0, B_1}$ for $f_1$.
Without loss of generality assume all component scores are positive, i.e., $X_{A_1}^0 >0, X_{B_0}^1>0, X_{B_1}^0 >0, X_{A_0}^1>0$ (if not we can shift the scores without changing the overall ranking). We have the following theorem for compositional fairness:
\begin{theorem}
If the following hold:
{\small
\begin{align}
(X_{A_1}^0 Z^1_{A_1, B_0} + X_{B_0}^1 Z^0_{A_1, B_0})^{(+)}=(X_{A_1}^0 Z^1_{A_1, B_0})^{(+)} + (X_{B_0}^1 Z^0_{A_1, B_0})^{(+)},
\label{eq:assumption1}
\end{align}
\begin{align}
(X_{B_1}^0 Z^1_{A_0, B_1} + X_{A_0}^1 Z^0_{A_0, B_1})^{(+)}=(X_{B_1}^0 Z^1_{A_0, B_1})^{(+)} + (X_{A_0}^1 Z^0_{A_0, B_1})^{(+)},
\label{eq:assumption2}
\end{align}
}%
where $Z^{(+)}=P(Z>0)$, then per-component fairness on $f_0$ and $f_1$ means we have compositional fairness for $f(x) = f_0(x) \cdot f_1(x)$.
\label{thm:pairwise}
\end{theorem}
\begin{proof}
With per-component fairness we have:
\begin{align*}
\text{ for } f_0, P(X_{A_1}^0 > X_{B_0}^0) &= P(X_{B_1}^0 > X_{A_0}^0), \text{ or } (Z^0_{A_1, B_0})^{(+)}=(Z^0_{A_0, B_1})^{(+)}, \\
\text{ for } f_1, P(X_{A_1}^1 > X_{B_0}^1) &= P(X_{B_1}^1 > X_{A_0}^1), \text{ or } (Z^1_{A_1, B_0})^{(+)}=(Z^1_{A_0, B_1})^{(+)},
\end{align*}
by composing $f_0$ and $f_1$ we have:
\begin{align*}
&(X_{A_1}^0 \cdot X_{A_1}^1 - X_{B_0}^0 \cdot X_{B_0}^1)^{(+)}\\
=&(X_{A_1}^0 \cdot X_{A_1}^1 - X_{A_1}^0 \cdot X_{B_0}^1 + X_{A_1}^0 \cdot X_{B_0}^1 - X_{B_0}^0 \cdot X_{B_0}^1)^{(+)}\\
=&(X_{A_1}^0 \cdot Z^1_{A_1, B_0} + X_{B_0}^1 \cdot Z^0_{A_1, B_0})^{(+)}\\
=&(X_{A_1}^0 \cdot Z^1_{A_1, B_0})^{(+)} + (X_{B_0}^1 \cdot Z^0_{A_1, B_0})^{(+)} 
\quad\text{(by Eq.~\eqref{eq:assumption1})}\\
=&(X_{A_1}^0 \cdot Z^1_{A_0, B_1})^{(+)} + (X_{B_0}^1 \cdot Z^0_{A_0, B_1})^{(+)} 
\quad\text{(by per-component fairness)}\\
=&(Z^1_{A_0, B_1})^{(+)} + (Z^0_{A_0, B_1})^{(+)} 
\qquad\qquad\quad\ \ \text{(by $X_{A_1}^0 >0, X_{B_0}^1>0$)}\\
=&(X_{B_1}^0 \cdot Z^1_{A_0, B_1})^{(+)} + (X_{A_0}^1 \cdot Z^0_{A_0, B_1})^{(+)} 
\quad\text{(by $X_{B_1}^0 >0, X_{A_0}^1>0$)}\\
=&(X_{B_1}^0 \cdot Z^1_{A_0, B_1} + X_{A_0}^1 \cdot Z^0_{A_0, B_1})^{(+)} 
\qquad\quad\text{(by Eq.~\eqref{eq:assumption2})}\\
=&(X_{B_1}^0 \cdot X_{B_1}^1 - X_{B_1}^0 \cdot X_{A_0}^1 + X_{B_1}^0 \cdot X_{A_0}^1 - X_{A_0}^0 \cdot X_{A_0}^1)^{(+)}\\
=&(X_{B_1}^0 \cdot X_{B_1}^1 - X_{A_0}^0 \cdot X_{A_0}^1)^{(+)},
\end{align*}
i.e., $P(X_{A_1}^0 \cdot X_{A_1}^1 > X_{B_0}^0 \cdot X_{B_0}^1) = P(X_{B_1}^0 \cdot X_{B_1}^1 > X_{A_0}^0 \cdot X_{A_0}^1)$, hence compositional fairness holds.
\end{proof}
One scenario for Eq.~\eqref{eq:assumption1} to hold is when there is a perfect match on the signs of $X_{A_1}^0 Z^1_{A_1, B_0}$ and $X_{B_0}^1 Z^0_{A_1, B_0}$ (similarly for Eq.~\eqref{eq:assumption2}). Alternatively, if the change of signs by comparing ($X_{A_1}^0 Z^1_{A_1, B_0}+X_{B_0}^1 Z^0_{A_1, B_0}$) to ($X_{A_1}^0 Z^1_{A_1, B_0}, X_{B_0}^1 Z^0_{A_1, B_0}$) adds up to zero, Eq.~\eqref{eq:assumption1} also holds.

\section{Analytical Framework and Modeling Solutions}
\label{sec:analytical}
As we see in the theoretical results, improving the fairness of individual components \emph{sometimes} benefits the fairness of the composite score, depending on the components and the relationship between them.  Therefore we ask: if we have a multi-component system where we observe fairness issues, how much will improving the fairness for each component help the overall system's fairness?

Taking this data-driven view of the problem, we find there are multiple questions that we can tractably answer:
\begin{enumerate}
    \item Given a system with fairness issues, improving which components would yield the greatest benefit?
    \item If all components were independently ``fixed'', what would be the resulting fairness metrics for the combined system?
\end{enumerate}

\subsection{Per-Component Fixes}
To discover which components would be most beneficial in improving fairness, we explore the impact of two realistic classes of methods.  

\subsubsection{Distribution matching for ranking exposure}
A significant amount of academic literature \cite{mmd,mcgan17} and publications on what is used practice \cite{beutel2019putting, DBLP:conf/kdd/BeutelCDQWWHZHC19}, takes the perspective of regularizing the model such that the distribution of predictions from each group (sometimes conditioned on the label) is matching.  Under different formulations this has been done by comparing the covariance \cite{pmlr-v54-zafar17a}, correlation \cite{beutel2019putting, DBLP:conf/kdd/BeutelCDQWWHZHC19}, and Maximum Mean Discrepancy \cite{mmd, MuaFukSriSch17} between the distributions.  Therefore, we consider whether matching the groups' distributions of predictions for each model has the desired effect on the combined fairness metrics.  

As this is an analytical framework, in contrast to a training framework, we can easily do this offline by directly changing the predictions over our dataset.  We consider distribution matching for a component $f_k$.  In order to match the distributions, we sort all examples in each group by their scores $f_k$.  
We define by $\mathbf{a}^{(k)}$ a sorted vector of scores for examples in Group $\setone$ and by $\phi_{a,k}$ the mapping of examples to positions in this sorted list, i.e. $\mathbf{a}^{(k)}_{\phi_{a,k}(x)} = f_k(x)$ and $\mathbf{a}^{(k)}_j \leq \mathbf{a}^{(k)}_{j+1}$ for all $j$; we similarly define $\mathbf{b}^{(k)}$ and $\phi_{b,k}$ for examples from Group $\settwo$.  
Assuming $|\setone| = |\settwo|$, when matching the distributions, we define the ``fixed'' component $\tilde{f}_k$ as follows:
\begin{align}
    \tilde{f}_k(x) = f_k(x), \mbox{ if } x \in \setone; 
    \quad\tilde{f}_k(x) = \mathbf{a}^{(k)}_{\phi_{b,k}(x)}, \mbox{ if } x \in \settwo.
\label{eq:dist_match}
\end{align}
That is, for examples in $\settwo$, $\tilde{f}_k$ returns the score for a similarly ranked item from $\setone$ such that the empirical distribution over $\setone$ and $\settwo$ exactly matches.
Note, $\mathbf{a}$ and $\mathbf{b}$ are the empirical cumulative distribution function (CDF) for $f_k$ over $\setone$ and $\settwo$ respectively, and as such if $|\setone| \neq |\settwo|$ then simple interpolation to match the empirical CDFs can be used.
\begin{theorem}
$\tilde{f}_k$ as defined by Eq. \eqref{eq:dist_match} has an exposure gap (defined by Eq. \eqref{eq:exposure_gap}) of zero, assuming the ranking order $r$ based on $\tilde{f}_k$ gives the exact same rank of $x$ given the same score of $\tilde{f}_k(x)$.\footnote{Note in real applications this is usually not the case since a tie-breaking strategy is needed for two items with the same score. Assuming the tie-breaking strategy is random, $\tilde{f}_k$ should achieve an exposure gap close to zero.}
\end{theorem}
\begin{proof}
Given the definition in Eq. \eqref{eq:exposure_gap}, it is easy to see that
\begin{align*}
\text{Gap}_{\text{exp}}(\mathcal{A}, \mathcal{B}|r) &= \frac{|\text{Exposure}(\mathcal{A}|r) - \text{Exposure}(\mathcal{B}|r)|}{\text{Exposure}(\mathcal{A}|r) + \text{Exposure}(\mathcal{B}|r)} \\
&= \frac{|\sum_{x\in \mathcal{A}} [\text{rank}(x)]^{-w} - \sum_{x\in \mathcal{B}} [\text{rank}(x)]^{-w}|}{\text{Exposure}(\mathcal{A}|r) + \text{Exposure}(\mathcal{B}|r)}= 0.
\end{align*}
Because there is an exact one-to-one correspondence of $\tilde{f}_k(x_i), x_i\in\mathcal{A}$ and $\tilde{f}_k(x_j), x_j\in \mathcal{B}$ that gives exactly one pair of $\text{rank}(x_i)=\text{rank}(x_j)$ from each group, which cancels each other given $|\setone| = |\settwo|$ and thus results in a zero gap.
\end{proof}

\subsubsection{Label-conditioned distribution matching for pairwise ranking accuracy}
\label{sec:label-conditioned-dm}
The above approach only recalibrates the predictions by group but does not necessarily align with any labels for the task.  As such, for the pairwise ranking accuracy gap metric (Eq.~\eqref{eq:pairwise_gap}) the method as described is not guaranteed to give per-component pairwise fairness.  For that, we provide a slight modification of the algorithm above, i.e., we exactly match the empirical distribution between $\Delta_k(\mathcal{A}_1, \mathcal{B}_0)$ and $\Delta_k(\mathcal{B}_1, \mathcal{A}_0)$ (Eq.~\eqref{eq:delta_pairwise_ranking_acc}), which aligns with the regularization proposed in \cite{DBLP:conf/kdd/BeutelCDQWWHZHC19}. In the following we show this label-conditioned distribution matching suffices to achieve pairwise ranking fairness for each component $k$.

\begin{theorem}
$\hat{f}_k$ as defined by matching $\Delta_k(\mathcal{A}_1, \mathcal{B}_0)$ and $\Delta_k(\mathcal{B}_1, \mathcal{A}_0)$ has a pairwise ranking accuracy gap, defined in Eq. \eqref{eq:pairwise_gap} of 0.
\label{thm:pairwise_gap}
\end{theorem}
\begin{proof}
The Pairwise Ranking Accuracy $\text{PairAcc}(\mathcal{A} > \mathcal{B}|r)$ is given by 
$
\frac{\sum_{x_i\in \mathcal{A}_1, x_j\in \mathcal{B}_0} I[\hat{f}_k(x_i) > \hat{f}_k(x_j)]}{|\mathcal{A}_1|\cdot |\mathcal{B}_0|}$,
i.e., it is equal to the percentage of positive deltas in 
$\hat{\Delta}_k(\mathcal{A}_1, \mathcal{B}_0)=\{\hat{f}_k(x_i) - \hat{f}_k(x_j) \mid x_i\in \mathcal{A}_1, x_j\in \mathcal{B}_0\}$ by definition.
Similarly, $\text{PairAcc}(\mathcal{B} > \mathcal{A}|r)$ is equal to the percentage of positive deltas in 
$\hat{\Delta}_k(\mathcal{B}_1, \mathcal{A}_0)$.

Given that we exactly matched the delta terms in $\hat{\Delta}_k(\mathcal{A}_1, \mathcal{B}_0)$ and $\hat{\Delta}_k(\mathcal{A}_1, \mathcal{B}_0)$, and since $|\setone_0| = |\setone_1| = |\settwo_0| = |\settwo_1|$, we have 
$\text{PairAcc}(\mathcal{A} > \mathcal{B}|r)
=\text{PairAcc}(\mathcal{B} > \mathcal{A}|r)$, i.e., the pairwise ranking accuracy gap (as defined in Eq. \eqref{eq:pairwise_gap}) is $0$.
\end{proof}

\subsubsection{Distribution Normalization}
While the above procedure is provably guaranteed to achieve per-component fairness, under the definitions given previously, in practice we would want to use a regularization on the model for this goal, which will be noisier.  As such, we consider a lighter-weight approach: per-group normalization:
\begin{definition}{Per-Group Normalization.}
We modify component $f_k$ to incorporate per-group normalization by:
\begin{align}
    \bar{f}_k(x) = \frac{f_k(x)-\mu_{x \in \mathcal{G}}[f_k(x)]}{\sigma_{x \in \mathcal{G}} [f_k(x)]}, \ \ \mathcal{G}\in \{\setone, \settwo\},
\label{eq:pg_norm}
\end{align}
\end{definition}
where $\mu, \sigma$ is the empirical mean and standard deviation on $f_k(x)$, for $x\in\setone$ or $x\in\settwo$, respectively.
While $\bar{f}$ is not guaranteed to provide even per-component fairness under either definition, we find in practice it too can significantly improve end-to-end fairness.

\subsection{Counterfactual Testing}
\label{sec:counterfactual_testing}
How can we use the modified functions described above to understand the system's end-to-end fairness properties?   All of the questions given at the beginning of this section are counterfactual questions: what would happen if we succeeded in fixing a component or set of components?  With the above methods for simulating a fixed component (without actually changing the model training), we can do this headroom analysis.

\paragraph{Per-Component Effect}
As before, we assume we have $K$ components which are multiplied together such that the overall score given to an example by the system is $f(x) = \prod_{k = 0}^{K-1} f_k(x)$.
Even when improving the fairness of one component, it is not guaranteed to improve the fairness of the overall system. For example, two components could be equally biased in opposite directions such that improving only one actually worsens the end-to-end fairness metrics.  

Therefore, we use the above per-component modifications to test the effect of independently improving individual components.  We will use $g_\kappa$ to characterize a modified component as described above, i.e., $g_\kappa \in \{\tilde{f}_\kappa, \hat{f}_\kappa, \bar{f}_\kappa\}$.  With this we can simulate how the system would behave if we improve a given component $\kappa$:
\vspace{-0.05in}
\begin{definition}[$\kappa$-Improved System]
Given a system $f$ with $K$ components $f_k$, and a simulated improved component $g_\kappa$ for component $\kappa$, we define the improved system as:
\begin{align}
    f^{(\bar{\kappa})}(x) = g_\kappa(x) \prod\nolimits_{k \neq \kappa} f_k(x) = f(x) \frac{g_\kappa(x)}{f_\kappa(x)}.
    \label{eq:one_fix}
\end{align}
\end{definition}
With this we can understand the fairness of this counterfactual system $f^{(\bar\kappa)}$, using either Eq.~\eqref{eq:exposure_gap} or \eqref{eq:pairwise_gap}: if we improved the fairness of component $\kappa$, what would be the resulting end-to-end fairness?

\subsection{Modeling Solution}
The above analytical framework is a low-cost way to test if per-component fix leads to compositional fairness before we implement any real algorithmic changes.
In practice, after we have identified that a system's overall fairness can be improved by per-component fixes, as well as which components should be prioritized to fix, we need an algorithmic change over the existing per-component model.
Within each component, it is easy to see that the problem is a standard single-task optimization problem, where the objective can in general be written as $L_{ranking} + \lambda \cdot L_{fairness}$, where $L_{ranking}$ is the original ranking loss (e.g., loss over CTR predictions), and $L_{fairness}$ is a fairness loss which has many instantiations in existing works \cite{pmlr-v54-zafar17a, beutel2019putting, prost19}.
In the experiment section we will explore modeling solutions and show that the results are consistent with the solutions we provided through our analytical framework.

\section{Experiments}

\subsection{Synthetic Data}
We begin with presenting experiments on synthetic datasets to demonstrate the relationship between per-component fairness and compositional fairness, connecting to our theoretical analysis.
Again we assume the system has two components $f_0$ and $f_1$, and we evaluate the fairness metrics with respect to two groups $\mathcal{A}$ and $\mathcal{B}$.

\paragraph{Dataset with independent Gaussian distributions.}
Assume 
\begin{align*}
f_0(x) \sim \mathcal{N}(10, 0.5), x\in \mathcal{A};&\ \ \ f_0(x) \sim \mathcal{N}(9, 0.5), x\in \mathcal{B}.\\
f_1(x) \sim \mathcal{N}(5, 0.5), x\in \mathcal{A};&\ \ \ f_1(x) \sim \mathcal{N}(4, 0.5), x\in \mathcal{B}. 
\end{align*}
We draw $1000$ examples from each group, as shown in Figure~\ref{fig:syndata_plot} (left), where the x-axis is for $f_0(x)$ and y-axis for is $f_1(x)$.
The two colors represent the two groups, respectively.

\begin{figure}[t]
\centering
\hspace{-0.2in}
    \includegraphics[width=1.7in]{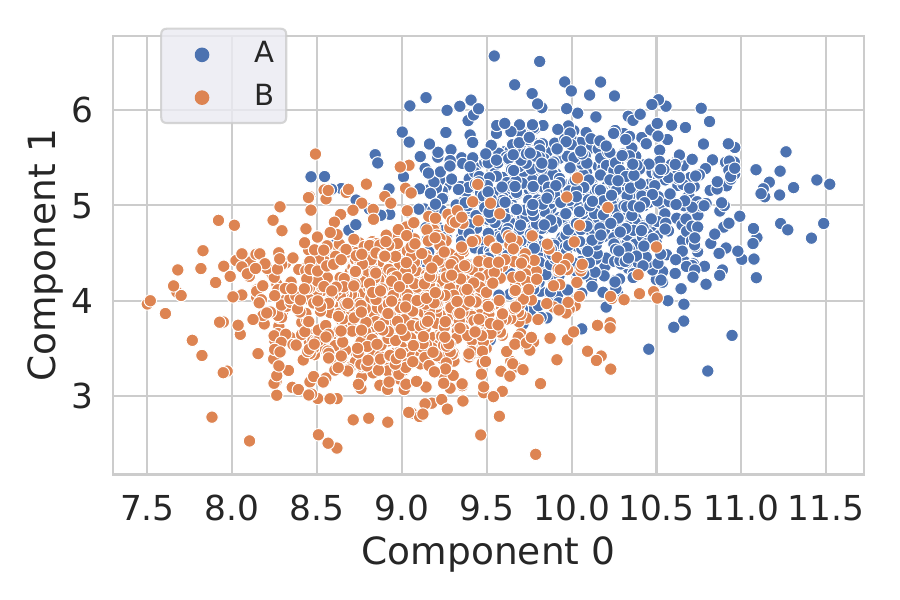}
    \includegraphics[width=1.7in]{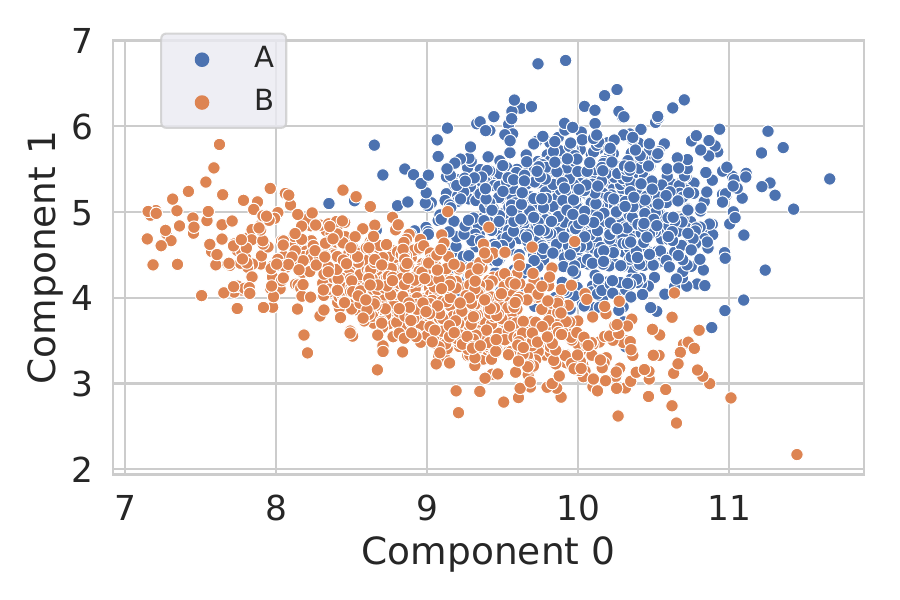}
    \hspace{-0.2in}
    \vspace{-0.15in}
    \caption{Data distribution for the two components on synthetic dataset 1 (left) and 2 (right).}
    \label{fig:syndata_plot}
\vspace{-0.2in}
\end{figure}

Table~\ref{table_syn_data} shows the \textit{ranking exposure} metric (as defined in Eq.~\eqref{eq:exposure_gap}, with $w=0.65$) on this synthetic dataset.  We see that Group $\mathcal{A}$ gets a significantly more exposure ($>50\%$ more) than Group $\mathcal{B}$.
We start by applying the fix on each component by \textit{distribution matching} (Eq.~\eqref{eq:dist_match}).
From Table~\ref{table_syn_data} we see that fixing only one component has very limited effect on overall system's fairness, and the end-to-end fairness can only be achieved by fixing both components (i.e., per-component fairness), which is consistent with Theorem~\ref{thm:exposure} since both distributions are symmetric and independent of each other.
Second, we test fixing each component by \textit{distribution normalization} (Eq.~\eqref{eq:pg_norm}), and it is also effective in reducing the gap between the two groups while fixing only one component at a time.

\begin{table*}
\small
\begin{center}
\begin{tabular}{|c|c|c|c||c|c|c|}
\hline
& \multicolumn{3}{c||}{Dataset 1} & \multicolumn{3}{c|}{Dataset 2} \\
\hline
 Fixed Component(s) & Group $\mathcal{A}$  & Group $\mathcal{B}$ & Overall Gap  & Group $\mathcal{A}$  & Group $\mathcal{B}$ & Overall Gap\\
\hline
None (baseline) & 0.7640 & 0.2360 & 0.5281  & 0.7699 & 0.2301 & 0.5398\\
\hline
\multicolumn{7}{|c|}{Distribution Matching}\\
\hline
Component 1 & 0.7433 & 0.2567 & 0.4865 & 0.7602 & 0.2398 & 0.5205\\
\hline
Component 2 & 0.6856 & 0.3144 & 0.3712 & 0.7318 & 0.2682 & 0.4636\\
\hline
Both & 0.4818 & 0.5182 & \textbf{-0.0365}\tablefootnote{The gap cannot be exactly 0 from discretization effects at the top of the list.} & 0.6262 & 0.3738 & \textbf{0.2524}\\
\hline
\multicolumn{7}{|c|}{Distribution Normalization}\\
\hline
Component 1 & 0.5472 & 0.4528 & 0.0943 & 0.6156 & 0.3844 & 0.2312\\
\hline
Component 2 & 0.5470 & 0.4530 & 0.0940 & 0.5765 & 0.4235 & \textbf{0.1529}\\
\hline
Both & 0.4858 & 0.5142 & \textbf{-0.0285} & 0.6950 & 0.3050 & 0.3899\\
\hline
\end{tabular}
\end{center}
\caption{The ranking exposure metric (over individual component and after composition) on Synthetic Dataset 1 and 2. Dataset 1 has independent distributions, and dataset 2 has anti-correlated distributions between components.}
\label{table_syn_data}
\vspace{-0.25in}
\end{table*}

\paragraph{Dataset with anti-correlated Gaussian distributions.}
In this experiment, we follow the same setting as the previous experiment except changing $f_0(x) = \mathcal{N}(13, 0.5) - f_1(x)$ for $x\in \mathcal{B}$ (we choose $\mu=13$ for the first Gaussian such that $\mu[f_0(x)]=9$, same as the first dataset) to create an anti-correlation between $f_0$ and $f_1$ for group $\mathcal{B}$.
Again $1000$ examples are sampled for each group, as shown in
Figure~\ref{fig:syndata_plot} (right). 
Table~\ref{table_syn_data} (right three columns) shows the fairness metrics on Synthetic Dataset 2. Compared with the results on Synthetic Dataset 1, we can see that the anti-correlation (thus breaking the symmetry requirement in Theorem~\ref{thm:exposure}) makes the end-to-end fairness metric much harder to achieve (larger overall gap).

\subsection{German Credit Data}
\begin{table}
\small
\begin{center}
\begin{tabular}{|c|c|c|c|}
\hline
Fixed Component(s) & Male Rep. & Female Rep. & Overall Gap\\
\hline
None (baseline) & 0.6081 & 0.3919 & 0.2162\\
\hline
credit amount  & 0.5852 &  0.4148 & 0.1704\\
\hline
age & 0.5865 & 0.4135 & 0.1731\\
\hline
num\_credits & 0.5986 & 0.4014 & 0.1972\\
\hline
num\_liable & 0.5953 & 0.4047 & 0.1907\\
\hline
credit amount \& age & 0.5652 & 0.4348 & 0.1304\\
\hline
\makecell{credit amount \& \\age \& num\_liable} & 0.5392 & 0.4608 & 0.0783\\
\hline
All components & 0.5352 & 0.4648 & 0.0705\\
\hline
\end{tabular}
\end{center}
\caption{Compositional fairness by distribution matching for each component on the German Credit dataset.}
\label{table_german_data}
\vspace{-0.25in}
\end{table}

In this section, we demonstrate our analytical framework on a public academic dataset: the German Credit data\footnote{https://archive.ics.uci.edu/ml/datasets/statlog+(german+credit+data)}, as another example to illustrate the effect of score composition on the end-to-end fairness.
This dataset provides a set of attributes for each person, including credit history, credit amount, installment rate, personal status, gender, age, etc., and the corresponding credit risk.

We assume the final score for assessing credit risk is composed by the following four factors:
1) credit amount; 2) age; 3) number of existing credits at this bank (denoted as ``num\_credits" in the following), 4) number of people being liable to provide maintenance for (denoted as ``num\_liable").
We consider the problem of ranking all people in this dataset by the above score composition, and we consider the end-to-end fairness metric to be the ranking exposure (Eq.~\eqref{eq:exposure_gap}, $w=0.65$) with respect to \textit{gender} groups: male, female\footnote{As that is how gender is categorized in the dataset.}.

In the first setting, we assume the same group size, i.e., $|\mathcal{A}|=|\mathcal{B}|$, and in order to achieve demographic parity, the top $N$ people within each gender group should receive the same ranking exposure.
In this case the ideal exposure gap should reach zero.
In Table~\ref{table_german_data}, we show the effect on the end-to-end fairness, in terms of the percentage of male and female representation in the end ranking, and the exposure gap between them.
We use \textit{distribution matching} (Eq.~\eqref{eq:dist_match}) to improve the system, and we use the counterfactual testing (Section~\ref{sec:counterfactual_testing}) to test the effect of fixing each component alone, and the effect of fixing different combinations of the components.

In Table~\ref{table_german_data} we see that distribution matching for each component independently can help on the compositional fairness (smaller ``Overall Gap") to different extents.
Fixing multiple components simultaneously effectively further improves compositional fairness, and the overall gap is reduced most when all components are fixed.
This headroom analysis provides us guidance on which components should be prioritized for improving end-to-end fairness.

\begin{figure}[h]
    \centering
    \hspace{-0.2in}
    \includegraphics[width=1.8in]{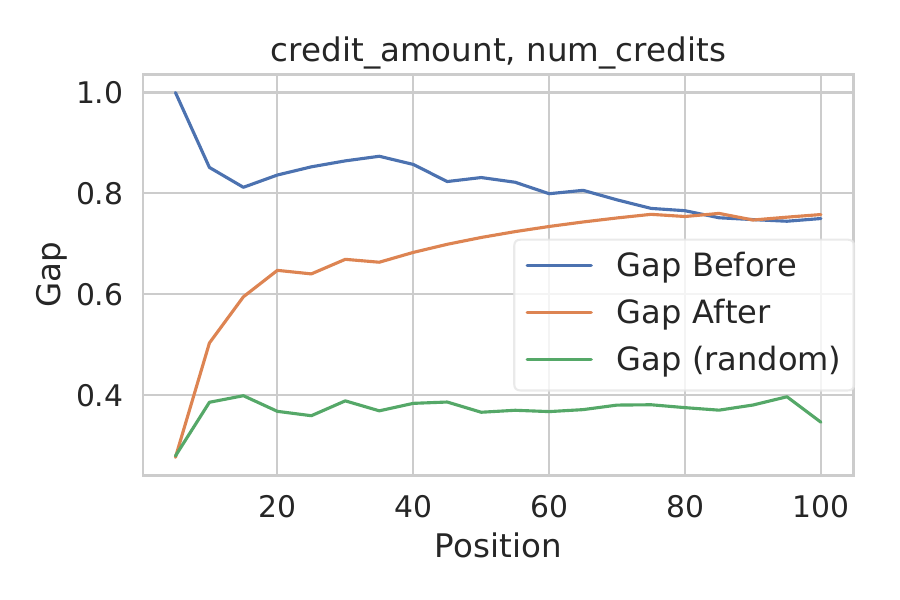}
    \hspace{-0.2in}
    \includegraphics[width=1.8in]{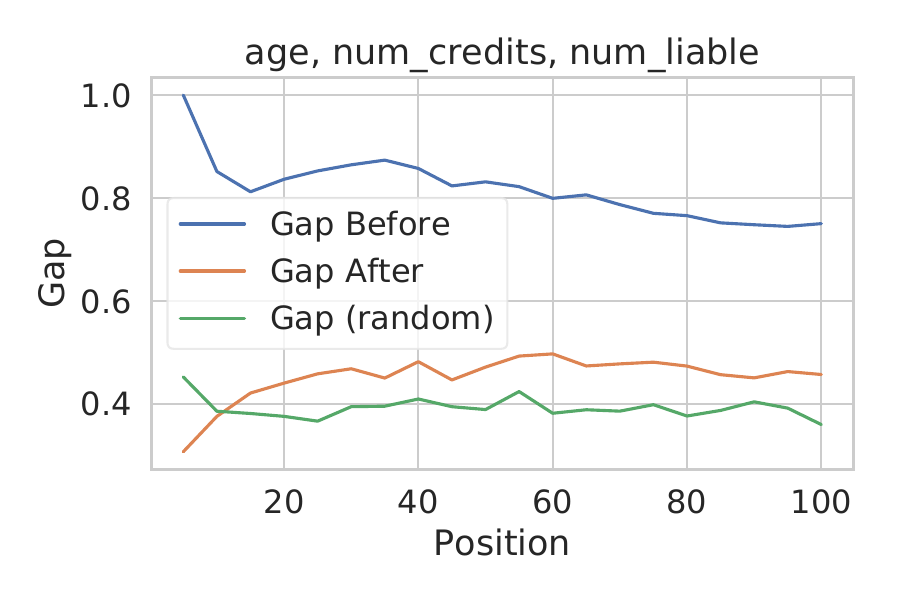}
    \hspace{-0.2in}
\vspace{-0.2in}
    \caption{Gap between gender groups with respect to each position, with different group sizes, by distribution normalization on the German Credit dataset.}
    \label{fig:multi_component_credit}
\vspace{-0.15in}
\end{figure}

\begin{table*}[h]
\begin{center}
\small
\begin{tabular}{|c|c|c|c|c|c|c|}
\hline
Fixed component(s) & Gap (CTR) & Gap (S1) & Gap (S2) & Group $\mathcal{A}$ Acc. & Group $\mathcal{B}$ Acc. & Overall Gap\\
\hline
None (baseline) & 0.1024 & 0.1781 & 0.1078 & 0.5862 & 0.7408 & 0.1546\\
\hline
\multicolumn{7}{|c|}{Matching on Marginal Distributions}\\
\hline
CTR & \textbf{0.0049} & 0.1781 & 0.1078 & 0.6198 & 0.7084 & 0.0886\\
\hline
Satisfaction 1 & 0.1024 & \textbf{0.0103} & 0.1078 & 0.6292 & 0.7054 & 0.0762\\
\hline
Satisfaction 2 & 0.1024 & 0.1781 & \textbf{0.0202} & 0.6021 & 0.7270 & 0.1248\\
\hline
All & 0.0049 & 0.0103 & 0.0202 & 0.6781 & 0.6546 & 0.0236\\
\hline
\multicolumn{7}{|c|}{Matching on Conditional Distributions}\\
\hline
CTR & \textbf{0.0057} & 0.1781 & 0.1078 & 0.6164 & 0.7048 & 0.0884\\
\hline
Satisfaction 1 & 0.1024 & \textbf{0.0092} & 0.1078 & 0.6258 & 0.7039 & 0.0781\\
\hline
Satisfaction 2 & 0.1024 & 0.1781 & \textbf{0.0197} & 0.6003 & 0.7258 & 0.1255\\
\hline
All & 0.0057 & 0.0092 & 0.0197 & 0.6697 & 0.6472 & 0.0225\\
\hline
\multicolumn{7}{|c|}{Matching on Delta Distributions}\\
\hline
CTR & \textbf{0.0000} & 0.1781 & 0.1078 & 0.6473 & 0.7408 & 0.0935\\
\hline
Satisfaction 1 & 0.1024 & \textbf{0.0000} & 0.1078 & 0.6669 & 0.7408 & 0.0739\\
\hline
Satisfaction 2 & 0.1024 & 0.1781 & \textbf{0.0000} & 0.6197 & 0.7408 & 0.1211\\
\hline
All & 0.0000 & 0.0000 & 0.0000 & 0.7630 & 0.7408 & \textbf{0.0222}\\
\hline
\end{tabular}
\end{center}
\caption{Compositional fairness by distribution matching in each component, on a large-scale real-world recommender system.}
\label{table_prod_system}
\vspace{-0.25in}
\end{table*}

In the second setting, we do not assume the same group size, and we rank $\mathcal{A}$ and $\mathcal{B}$ proportional to their respective sizes ($|\mathcal{A}|=690$ for male, $|\mathcal{B}|=310$ for female on this dataset). We vary the number of top positions $t$ and plot the exposure gap metric with respect to the positions. 
As a reference, we also plot the exposure gap under random ordering (denoted as ``Gap (random)" in the figures, by averaging over $100$ runs), which ranks each person regardless of their gender.
In Figure~\ref{fig:multi_component_credit}, we show the results by \textit{distribution normalization} on multiple components simultaneously.
The title of each sub-figure indicates the components that we have applied distribution normalization on.
We can see that under this more realistic setting, our proposed fixes over per-component (and over combinations of components) can still effectively lead to significant improvements on the end-to-end fairness metric.

Empirically the results indicate that fairness does compose approximately on this dataset, we further analyzed the data distribution and found that most of the component scores (and the sum of many component combinations) are Gaussian-distributed (and thus symmetric), hence following Theorem~\ref{thm:exposure} compositional fairness holds relatively well on this dataset.

\subsection{Case Study on A Real Production Recommender System}
In this section, we describe the results on a large-scale real-world recommender system.
On an abstract level, the system mainly consists of three different components, one predicting the probability of click (denoted as ``CTR''), and two other components predicting different signals of user satisfaction, denoted as ``Satisfaction 1 (S1)'' and ``Satisfaction 2 (S2)'', similar to \cite{DBLP:conf/kdd/BeutelCDQWWHZHC19}.
In the following, we present results by fixing each individual component by:
\begin{itemize}
    \item Matching the marginal distribution of $f_k(x), x\in\mathcal{A}$ and $f_k(x), x\in\mathcal{B}$, as in Eq. \eqref{eq:dist_match}. 
    \item Matching the conditional distributions of $X^k_{A_0}$ and $X^k_{B_0}$, and the conditional distributions of $X^k_{A_1}$ and $X^k_{B_1}$, building on \eqref{eq:dist_match}. 
    \item Matching the distribution on the delta terms: $\Delta_k(\mathcal{A}_1, \mathcal{B}_0)$ and $\Delta_k(\mathcal{B}_1, \mathcal{A}_0)$, as in Theorem~\ref{thm:pairwise_gap}.
\end{itemize}
The results are shown in Table~\ref{table_prod_system}, the first column denotes the ``fixed" component(s), and column 2-4 show the Pairwise Ranking Accuracy Gap (Eq.~\eqref{eq:pairwise_gap}) for each component (abbreviated to ``CTR", ``S1", ``S2"), respectively. Column 5-7 show the overall (compositional) Pairwise Ranking Accuracy for Group $\mathcal{A}$ and $\mathcal{B}$ and the overall Pairwise Ranking Accuracy Gap.
We see that first, compared to matching on marginal/conditional distributions, matching on the delta distributions is the only method that achieves zero gap on the per-component gap metric (Column 2-4 in Table~\ref{table_prod_system}). This is consistent with our theory (Theorem~\ref{thm:pairwise_gap}).
Second, although marginal/conditional distribution matching does not provably ensure per-component fairness, empirically they still lead to a significant gap reduction (all close to zero), and effectively help on the compositional fairness (Column ``Overall Gap" in Table~\ref{table_prod_system}).
Finally, compositional fairness is better achieved when all the components are fixed, and fixing per-component alone only helps to a certain extent on the compositional fairness. 

Analysis over the component scores shows that in this system, most components are correlated with each other, hence it is easier for the conditions in Theorem~\ref{thm:pairwise} to hold. This is consistent with the results in Table~\ref{table_prod_system} where we see that fairness composes well after making each component independently fair.

\subsection{Improving the Real Recommender}
Last, we test if applying model regularization approaches to each component successfully improves the fairness of the overall recommender system.
We implemented the fairness loss similar to \cite{prost19}, since it is most similar to our offline label-conditioned distribution matching analysis (Section~\ref{sec:label-conditioned-dm}).
Specifically, we minimize the Maximum Mean Discrepancy (MMD) \cite{mmd} between the two delta distributions: $\Delta_k(\mathcal{A}_1, \mathcal{B}_0)$ and $\Delta_k(\mathcal{B}_1, \mathcal{A}_0)$.
The results are shown in Table~\ref{tab:modeling}\footnote{Note, as these results are from continuous training, the dataset and thus absolute numbers differ slightly from the previous analysis.}. 
As we see there, applying the MMD regularization to any individual component leaves a significant pairwise accuracy gap, but by applying it to all of the components (still independently to each one) significantly reduces the pairwise ranking accuracy gap.
This is highly encouraging in that it demonstrates that significant progress can be made on fairness even in multi-component recommender systems and suggests that the insights gained from our counterfactual testing could be very useful in deciding when per-component fixes should be applied.
\begin{table}[]
\setlength\tabcolsep{3pt}
    \centering
    \begin{tabular}{c|c|c|c|c}
        Fixed components    & CTR & S1 &  S2 & All\\
                        \hline
        Distribution matching via MMD \cite{prost19} & 0.080 & 0.088 & 0.067 & 0.013   \\
    \end{tabular}
    \caption{Pairwise ranking accuracy gap on the real-world recommender system from modeling-based approaches (label-conditioned MMD regularization). }
    \label{tab:modeling}
\vspace{-0.3in}
\end{table}

\section{Conclusion}

As most industrial recommender systems are composed of many models and tasks, understanding how and when fairness composes is crucially important to enabling the application of fairness principles in practice.
In this paper, we explore, both theoretically and empirically, the question: given a recommender system where the end ranking score is the product of scores from each component, does making each component fair independently improve the full system's fairness?
We formalize this problem in two recently proposed fairness metrics for ranking, \textit{fairness of exposure}, and \textit{pairwise ranking accuracy gap}. 
For both metrics we find the unfortunate challenge that recommender fairness is not guaranteed to compose, aligned with prior work in classifiers \cite{dwork2018fairness}.

To overcome this obstacle, we focus on studying when \emph{does} fairness compose.
We first present theory showing conditions under which per-component fairness does compose.
Because the theory shows that composition is distribution-dependent, we propose taking a data-driven approach to this problem.  We offer an analytical framework for measuring how much improving per-component fairness will improve the recommender system's fairness and diagnosing which components we should prioritize improving to have the most impact.
By applying our analytical framework to multiple datasets, including a large real-world recommender system, we find that in practice most of the end-to-end exposure or accuracy gaps can be addressed through independently applying  per-component improvements.

Our results highlight that while guarantees do not hold in the worst-case, there is more nuance over realistic data distributions.  We believe this framework is a strong foundation for future research on compositional fairness, with clear extensions for classification \cite{kallus2019fairness,borkan2019nuanced,narasimhan2019pairwise} and opportunities to generalize to more fairness metrics and compositional functional forms.

\bibliographystyle{ACM-Reference-Format}
\bibliography{main}


\begin{thebibliography}{58}


\ifx \showCODEN    \undefined \def \showCODEN     #1{\unskip}     \fi
\ifx \showDOI      \undefined \def \showDOI       #1{#1}\fi
\ifx \showISBNx    \undefined \def \showISBNx     #1{\unskip}     \fi
\ifx \showISBNxiii \undefined \def \showISBNxiii  #1{\unskip}     \fi
\ifx \showISSN     \undefined \def \showISSN      #1{\unskip}     \fi
\ifx \showLCCN     \undefined \def \showLCCN      #1{\unskip}     \fi
\ifx \shownote     \undefined \def \shownote      #1{#1}          \fi
\ifx \showarticletitle \undefined \def \showarticletitle #1{#1}   \fi
\ifx \showURL      \undefined \def \showURL       {\relax}        \fi
\providecommand\bibfield[2]{#2}
\providecommand\bibinfo[2]{#2}
\providecommand\natexlab[1]{#1}
\providecommand\showeprint[2][]{arXiv:#2}

\bibitem[\protect\citeauthoryear{Adomavicius and Tuzhilin}{Adomavicius and
  Tuzhilin}{2005}]%
        {survey05}
\bibfield{author}{\bibinfo{person}{Gediminas Adomavicius} {and}
  \bibinfo{person}{Alexander Tuzhilin}.} \bibinfo{year}{2005}\natexlab{}.
\newblock \showarticletitle{Toward the Next Generation of Recommender Systems:
  A Survey of the State-of-the-Art and Possible Extensions}. In
  \bibinfo{booktitle}{\emph{TKDE}}.
\newblock


\bibitem[\protect\citeauthoryear{Agarwal, Beygelzimer, Dud{\'{\i}}k, Langford,
  and Wallach}{Agarwal et~al\mbox{.}}{2018}]%
        {DBLP:conf/icml/AgarwalBD0W18}
\bibfield{author}{\bibinfo{person}{Alekh Agarwal}, \bibinfo{person}{Alina
  Beygelzimer}, \bibinfo{person}{Miroslav Dud{\'{\i}}k}, \bibinfo{person}{John
  Langford}, {and} \bibinfo{person}{Hanna~M. Wallach}.}
  \bibinfo{year}{2018}\natexlab{}.
\newblock \showarticletitle{A Reductions Approach to Fair Classification}. In
  \bibinfo{booktitle}{\emph{ICML}}.
\newblock


\bibitem[\protect\citeauthoryear{Agrawal, Gollapudi, Halverson, and
  Ieong}{Agrawal et~al\mbox{.}}{2009}]%
        {diverse_wsdm09}
\bibfield{author}{\bibinfo{person}{Rakesh Agrawal}, \bibinfo{person}{Sreenivas
  Gollapudi}, \bibinfo{person}{Alan Halverson}, {and} \bibinfo{person}{Samuel
  Ieong}.} \bibinfo{year}{2009}\natexlab{}.
\newblock \showarticletitle{Diversifying Search Results}. In
  \bibinfo{booktitle}{\emph{WSDM}}.
\newblock


\bibitem[\protect\citeauthoryear{Barbieri, Silvestri, and Lalmas}{Barbieri
  et~al\mbox{.}}{2016}]%
        {barbieri2016improving}
\bibfield{author}{\bibinfo{person}{Nicola Barbieri}, \bibinfo{person}{Fabrizio
  Silvestri}, {and} \bibinfo{person}{Mounia Lalmas}.}
  \bibinfo{year}{2016}\natexlab{}.
\newblock \showarticletitle{Improving post-click user engagement on native ads
  via survival analysis}. In \bibinfo{booktitle}{\emph{Proceedings of the 25th
  International Conference on World Wide Web}}. \bibinfo{pages}{761--770}.
\newblock


\bibitem[\protect\citeauthoryear{Bennett, Lanning, et~al\mbox{.}}{Bennett
  et~al\mbox{.}}{2007}]%
        {bennett2007netflix}
\bibfield{author}{\bibinfo{person}{James Bennett}, \bibinfo{person}{Stan
  Lanning}, {et~al\mbox{.}}} \bibinfo{year}{2007}\natexlab{}.
\newblock \showarticletitle{The netflix prize}. In
  \bibinfo{booktitle}{\emph{Proceedings of KDD cup and workshop}},
  Vol.~\bibinfo{volume}{2007}. New York, \bibinfo{pages}{35}.
\newblock


\bibitem[\protect\citeauthoryear{Beutel, Chen, Doshi, Qian, Wei, Wu, Heldt,
  Zhao, Hong, Chi, and Goodrow}{Beutel et~al\mbox{.}}{2019a}]%
        {DBLP:conf/kdd/BeutelCDQWWHZHC19}
\bibfield{author}{\bibinfo{person}{Alex Beutel}, \bibinfo{person}{Jilin Chen},
  \bibinfo{person}{Tulsee Doshi}, \bibinfo{person}{Hai Qian},
  \bibinfo{person}{Li Wei}, \bibinfo{person}{Yi Wu}, \bibinfo{person}{Lukasz
  Heldt}, \bibinfo{person}{Zhe Zhao}, \bibinfo{person}{Lichan Hong},
  \bibinfo{person}{Ed~H. Chi}, {and} \bibinfo{person}{Cristos Goodrow}.}
  \bibinfo{year}{2019}\natexlab{a}.
\newblock \showarticletitle{Fairness in Recommendation Ranking through Pairwise
  Comparisons}. In \bibinfo{booktitle}{\emph{{KDD}' 19}}.
\newblock


\bibitem[\protect\citeauthoryear{Beutel, Chen, Doshi, Qian, Woodruff, Luu,
  Kreitmann, Bischof, and Chi}{Beutel et~al\mbox{.}}{2019b}]%
        {beutel2019putting}
\bibfield{author}{\bibinfo{person}{Alex Beutel}, \bibinfo{person}{Jilin Chen},
  \bibinfo{person}{Tulsee Doshi}, \bibinfo{person}{Hai Qian},
  \bibinfo{person}{Allison Woodruff}, \bibinfo{person}{Christine Luu},
  \bibinfo{person}{Pierre Kreitmann}, \bibinfo{person}{Jonathan Bischof}, {and}
  \bibinfo{person}{Ed~H Chi}.} \bibinfo{year}{2019}\natexlab{b}.
\newblock \showarticletitle{Putting Fairness Principles into Practice:
  Challenges, Metrics, and Improvements}.
\newblock \bibinfo{journal}{\emph{arXiv preprint arXiv:1901.04562}}
  (\bibinfo{year}{2019}).
\newblock


\bibitem[\protect\citeauthoryear{Beutel, Chen, Zhao, and H.~Chi}{Beutel
  et~al\mbox{.}}{2017}]%
        {beutel17}
\bibfield{author}{\bibinfo{person}{Alex Beutel}, \bibinfo{person}{Jilin Chen},
  \bibinfo{person}{Zhe Zhao}, {and} \bibinfo{person}{Ed H.~Chi}.}
  \bibinfo{year}{2017}\natexlab{}.
\newblock \showarticletitle{Data Decisions and Theoretical Implications when
  Adversarially Learning Fair Representations}. In
  \bibinfo{booktitle}{\emph{2017 Workshop on Fairness, Accountability, and
  Transparency in Machine Learning}}.
\newblock


\bibitem[\protect\citeauthoryear{Biega, Gummadi, and Weikum}{Biega
  et~al\mbox{.}}{2018}]%
        {biega2018equity}
\bibfield{author}{\bibinfo{person}{Asia~J Biega}, \bibinfo{person}{Krishna~P
  Gummadi}, {and} \bibinfo{person}{Gerhard Weikum}.}
  \bibinfo{year}{2018}\natexlab{}.
\newblock \showarticletitle{Equity of attention: Amortizing individual fairness
  in rankings}. In \bibinfo{booktitle}{\emph{SIGIR}}.
  \bibinfo{pages}{405--414}.
\newblock


\bibitem[\protect\citeauthoryear{Bolukbasi, Chang, Zou, Saligrama, and
  Kalai}{Bolukbasi et~al\mbox{.}}{2016}]%
        {tolga2016}
\bibfield{author}{\bibinfo{person}{Tolga Bolukbasi}, \bibinfo{person}{Kai-Wei
  Chang}, \bibinfo{person}{James Zou}, \bibinfo{person}{Venkatesh Saligrama},
  {and} \bibinfo{person}{Adam Kalai}.} \bibinfo{year}{2016}\natexlab{}.
\newblock \showarticletitle{Man is to Computer Programmer as Woman is to
  Homemaker? Debiasing Word Embeddings}. In \bibinfo{booktitle}{\emph{NIPS
  2016}}.
\newblock


\bibitem[\protect\citeauthoryear{Borkan, Dixon, Sorensen, Thain, and
  Vasserman}{Borkan et~al\mbox{.}}{2019}]%
        {borkan2019nuanced}
\bibfield{author}{\bibinfo{person}{Daniel Borkan}, \bibinfo{person}{Lucas
  Dixon}, \bibinfo{person}{Jeffrey Sorensen}, \bibinfo{person}{Nithum Thain},
  {and} \bibinfo{person}{Lucy Vasserman}.} \bibinfo{year}{2019}\natexlab{}.
\newblock \showarticletitle{Nuanced Metrics for Measuring Unintended Bias with
  Real Data for Text Classification}. In \bibinfo{booktitle}{\emph{Companion
  Proceedings of The 2019 World Wide Web Conference}}. ACM,
  \bibinfo{pages}{491--500}.
\newblock


\bibitem[\protect\citeauthoryear{Burke}{Burke}{2002}]%
        {hybrid02}
\bibfield{author}{\bibinfo{person}{Robin Burke}.}
  \bibinfo{year}{2002}\natexlab{}.
\newblock \showarticletitle{Hybrid Recommender Systems: Survey and
  Experiments}. In \bibinfo{booktitle}{\emph{User Modeling and User-Adapted
  Interaction, Volume 12, Issue 4}}. \bibinfo{pages}{331--370}.
\newblock


\bibitem[\protect\citeauthoryear{Burke}{Burke}{2017}]%
        {burke2017multisided}
\bibfield{author}{\bibinfo{person}{Robin Burke}.}
  \bibinfo{year}{2017}\natexlab{}.
\newblock \showarticletitle{Multisided fairness for recommendation}.
\newblock \bibinfo{journal}{\emph{arXiv preprint arXiv:1707.00093}}
  (\bibinfo{year}{2017}).
\newblock


\bibitem[\protect\citeauthoryear{Calders, Kamiran, and Pechenizkiy}{Calders
  et~al\mbox{.}}{2009}]%
        {Calders:2009:BCI:1674660.1677211}
\bibfield{author}{\bibinfo{person}{Toon Calders}, \bibinfo{person}{Faisal
  Kamiran}, {and} \bibinfo{person}{Mykola Pechenizkiy}.}
  \bibinfo{year}{2009}\natexlab{}.
\newblock \showarticletitle{Building Classifiers with Independency
  Constraints}. In \bibinfo{booktitle}{\emph{ICDMW '09}}.
\newblock


\bibitem[\protect\citeauthoryear{Capannini, Nardini, Perego, and
  Silvestri}{Capannini et~al\mbox{.}}{2011}]%
        {diverse_vldb11}
\bibfield{author}{\bibinfo{person}{Gabriele Capannini},
  \bibinfo{person}{Franco~Maria Nardini}, \bibinfo{person}{Raffaele Perego},
  {and} \bibinfo{person}{Fabrizio Silvestri}.} \bibinfo{year}{2011}\natexlab{}.
\newblock \showarticletitle{Efficient Diversification of Web Search Results}.
  In \bibinfo{booktitle}{\emph{VLDB}}.
\newblock


\bibitem[\protect\citeauthoryear{Carbonell and Goldstein}{Carbonell and
  Goldstein}{1998}]%
        {diverse_sigir98}
\bibfield{author}{\bibinfo{person}{J. Carbonell} {and} \bibinfo{person}{J.
  Goldstein}.} \bibinfo{year}{1998}\natexlab{}.
\newblock \showarticletitle{The use of MMR, diversity-based reranking for
  reordering documents and producing summaries}. In
  \bibinfo{booktitle}{\emph{Proceedings of SIGIR'98}}.
\newblock


\bibitem[\protect\citeauthoryear{Celis, Straszak, and Vishnoi}{Celis
  et~al\mbox{.}}{2018}]%
        {celis2018ranking}
\bibfield{author}{\bibinfo{person}{L~Elisa Celis}, \bibinfo{person}{Damian
  Straszak}, {and} \bibinfo{person}{Nisheeth~K Vishnoi}.}
  \bibinfo{year}{2018}\natexlab{}.
\newblock \showarticletitle{Ranking with Fairness Constraints}. In
  \bibinfo{booktitle}{\emph{45th International Colloquium on Automata,
  Languages, and Programming (ICALP 2018)}}. Schloss Dagstuhl-Leibniz-Zentrum
  fuer Informatik.
\newblock


\bibitem[\protect\citeauthoryear{Claypool, Gokhale, Miranda, Murnikov, Netes,
  and Sartin}{Claypool et~al\mbox{.}}{1999}]%
        {claypool1999combing}
\bibfield{author}{\bibinfo{person}{Mark Claypool}, \bibinfo{person}{Anuja
  Gokhale}, \bibinfo{person}{Tim Miranda}, \bibinfo{person}{Paul Murnikov},
  \bibinfo{person}{Dmitry Netes}, {and} \bibinfo{person}{Matthew Sartin}.}
  \bibinfo{year}{1999}\natexlab{}.
\newblock \showarticletitle{Combing content-based and collaborative filters in
  an online newspaper}.
\newblock \bibinfo{journal}{\emph{ACM SIGIR'99. Workshop on Recommender
  Systems: Algorithms and Evaluation}} (\bibinfo{year}{1999}).
\newblock


\bibitem[\protect\citeauthoryear{Cotter, Friedlander, Goh, and Gupta}{Cotter
  et~al\mbox{.}}{2016}]%
        {gupta16}
\bibfield{author}{\bibinfo{person}{Andrew Cotter}, \bibinfo{person}{Michael
  Friedlander}, \bibinfo{person}{Gabriel Goh}, {and} \bibinfo{person}{Maya
  Gupta}.} \bibinfo{year}{2016}\natexlab{}.
\newblock \showarticletitle{Satisfying Real-world Goals with Dataset
  Constraints}.
\newblock  (\bibinfo{year}{2016}).
\newblock


\bibitem[\protect\citeauthoryear{Cotter, Gupta, Jiang, Srebro, Sridharan, Wang,
  Woodworth, and You}{Cotter et~al\mbox{.}}{2019}]%
        {gupta19}
\bibfield{author}{\bibinfo{person}{Andrew Cotter}, \bibinfo{person}{Maya
  Gupta}, \bibinfo{person}{Heinrich Jiang}, \bibinfo{person}{Nathan Srebro},
  \bibinfo{person}{Karthik Sridharan}, \bibinfo{person}{Serena Wang},
  \bibinfo{person}{Blake Woodworth}, {and} \bibinfo{person}{Seungil You}.}
  \bibinfo{year}{2019}\natexlab{}.
\newblock \showarticletitle{Training Well-Generalizing Classifiers for Fairness
  Metrics and Other Data-Dependent Constraints}. In
  \bibinfo{booktitle}{\emph{ICML}}.
\newblock


\bibitem[\protect\citeauthoryear{Davidson, Liebald, Liu, Nandy, Van~Vleet,
  Gargi, Gupta, He, Lambert, Livingston, and et~al.}{Davidson
  et~al\mbox{.}}{2010}]%
        {youtube2010}
\bibfield{author}{\bibinfo{person}{James Davidson}, \bibinfo{person}{Benjamin
  Liebald}, \bibinfo{person}{Junning Liu}, \bibinfo{person}{Palash Nandy},
  \bibinfo{person}{Taylor Van~Vleet}, \bibinfo{person}{Ullas Gargi},
  \bibinfo{person}{Sujoy Gupta}, \bibinfo{person}{Yu He}, \bibinfo{person}{Mike
  Lambert}, \bibinfo{person}{Blake Livingston}, {and} \bibinfo{person}{et al.}}
  \bibinfo{year}{2010}\natexlab{}.
\newblock \showarticletitle{The YouTube Video Recommendation System}. In
  \bibinfo{booktitle}{\emph{RecSys}}.
\newblock


\bibitem[\protect\citeauthoryear{Diaz, Mitra, Ekstrand, Biega, and
  Carterette}{Diaz et~al\mbox{.}}{2020}]%
        {diaz2020evaluating}
\bibfield{author}{\bibinfo{person}{Fernando Diaz}, \bibinfo{person}{Bhaskar
  Mitra}, \bibinfo{person}{Michael~D Ekstrand}, \bibinfo{person}{Asia~J Biega},
  {and} \bibinfo{person}{Ben Carterette}.} \bibinfo{year}{2020}\natexlab{}.
\newblock \showarticletitle{Evaluating Stochastic Rankings with Expected
  Exposure}.
\newblock \bibinfo{journal}{\emph{arXiv preprint arXiv:2004.13157}}
  (\bibinfo{year}{2020}).
\newblock


\bibitem[\protect\citeauthoryear{Dwork and Ilvento}{Dwork and Ilvento}{2018a}]%
        {dwork2018fairness}
\bibfield{author}{\bibinfo{person}{Cynthia Dwork} {and}
  \bibinfo{person}{Christina Ilvento}.} \bibinfo{year}{2018}\natexlab{a}.
\newblock \showarticletitle{Fairness under composition}. In
  \bibinfo{booktitle}{\emph{arXiv preprint arXiv:1806.06122}}.
\newblock


\bibitem[\protect\citeauthoryear{Dwork and Ilvento}{Dwork and Ilvento}{2018b}]%
        {cynthia2018}
\bibfield{author}{\bibinfo{person}{Cynthia Dwork} {and}
  \bibinfo{person}{Christina Ilvento}.} \bibinfo{year}{2018}\natexlab{b}.
\newblock \showarticletitle{Group fairness under composition}. In
  \bibinfo{booktitle}{\emph{FATML}}.
\newblock


\bibitem[\protect\citeauthoryear{Dwork, Ilvento, and Jagadeesan}{Dwork
  et~al\mbox{.}}{2020}]%
        {dwork2020individual}
\bibfield{author}{\bibinfo{person}{Cynthia Dwork}, \bibinfo{person}{Christina
  Ilvento}, {and} \bibinfo{person}{Meena Jagadeesan}.}
  \bibinfo{year}{2020}\natexlab{}.
\newblock \showarticletitle{Individual Fairness in Pipelines}.
\newblock \bibinfo{journal}{\emph{arXiv preprint arXiv:2004.05167}}
  (\bibinfo{year}{2020}).
\newblock


\bibitem[\protect\citeauthoryear{Ekstrand, Tian, Kazi, Mehrpouyan, and
  Kluver}{Ekstrand et~al\mbox{.}}{2018}]%
        {fairness_book18}
\bibfield{author}{\bibinfo{person}{Michael~D. Ekstrand}, \bibinfo{person}{Mucun
  Tian}, \bibinfo{person}{Mohammed R.~Imran Kazi}, \bibinfo{person}{Hoda
  Mehrpouyan}, {and} \bibinfo{person}{Daniel Kluver}.}
  \bibinfo{year}{2018}\natexlab{}.
\newblock \showarticletitle{Exploring Author Gender in Book Rating and
  Recommendation}. In \bibinfo{booktitle}{\emph{RecSys '18}}.
  \bibinfo{pages}{242--250}.
\newblock


\bibitem[\protect\citeauthoryear{Garcia-Gathright, St.~Thomas, Hosey, Nazari,
  and Diaz}{Garcia-Gathright et~al\mbox{.}}{2018}]%
        {garcia2018understanding}
\bibfield{author}{\bibinfo{person}{Jean Garcia-Gathright},
  \bibinfo{person}{Brian St.~Thomas}, \bibinfo{person}{Christine Hosey},
  \bibinfo{person}{Zahra Nazari}, {and} \bibinfo{person}{Fernando Diaz}.}
  \bibinfo{year}{2018}\natexlab{}.
\newblock \showarticletitle{Understanding and evaluating user satisfaction with
  music discovery}. In \bibinfo{booktitle}{\emph{SIGIR}}.
  \bibinfo{pages}{55--64}.
\newblock


\bibitem[\protect\citeauthoryear{Geyik, Ambler, and Kenthapadi}{Geyik
  et~al\mbox{.}}{2019}]%
        {DBLP:conf/kdd/GeyikAK19}
\bibfield{author}{\bibinfo{person}{Sahin~Cem Geyik}, \bibinfo{person}{Stuart
  Ambler}, {and} \bibinfo{person}{Krishnaram Kenthapadi}.}
  \bibinfo{year}{2019}\natexlab{}.
\newblock \showarticletitle{Fairness-Aware Ranking in Search {\&}
  Recommendation Systems with Application to LinkedIn Talent Search}. In
  \bibinfo{booktitle}{\emph{KDD '19}}. \bibinfo{pages}{2221--2231}.
\newblock


\bibitem[\protect\citeauthoryear{Gollapudi and Sharma}{Gollapudi and
  Sharma}{2009}]%
        {diverse_www09}
\bibfield{author}{\bibinfo{person}{Sreenivas Gollapudi} {and}
  \bibinfo{person}{Aneesh Sharma}.} \bibinfo{year}{2009}\natexlab{}.
\newblock \showarticletitle{An Axiomatic Approach for Result Diversification}.
  In \bibinfo{booktitle}{\emph{WWW}}.
\newblock


\bibitem[\protect\citeauthoryear{Gretton, Borgwardt, Rasch, Sch\"{o}lkopf, and
  Smola}{Gretton et~al\mbox{.}}{2012}]%
        {mmd}
\bibfield{author}{\bibinfo{person}{Arthur Gretton}, \bibinfo{person}{Karsten~M.
  Borgwardt}, \bibinfo{person}{Malte~J. Rasch}, \bibinfo{person}{Bernhard
  Sch\"{o}lkopf}, {and} \bibinfo{person}{Alexander Smola}.}
  \bibinfo{year}{2012}\natexlab{}.
\newblock \showarticletitle{A kernel two-sample test}. In
  \bibinfo{booktitle}{\emph{The Journal of Machine Learning Research, Volume
  13}}. \bibinfo{pages}{723--773}.
\newblock


\bibitem[\protect\citeauthoryear{Hardt, Price, Srebro, et~al\mbox{.}}{Hardt
  et~al\mbox{.}}{2016}]%
        {hardt2016equality}
\bibfield{author}{\bibinfo{person}{Moritz Hardt}, \bibinfo{person}{Eric Price},
  \bibinfo{person}{Nati Srebro}, {et~al\mbox{.}}}
  \bibinfo{year}{2016}\natexlab{}.
\newblock \showarticletitle{Equality of opportunity in supervised learning}. In
  \bibinfo{booktitle}{\emph{Advances in neural information processing
  systems}}.
\newblock


\bibitem[\protect\citeauthoryear{He, Pan, Jin, Xu, Liu, Xu, Shi, Atallah,
  Herbrich, Bowers, et~al\mbox{.}}{He et~al\mbox{.}}{2014}]%
        {he2014practical}
\bibfield{author}{\bibinfo{person}{Xinran He}, \bibinfo{person}{Junfeng Pan},
  \bibinfo{person}{Ou Jin}, \bibinfo{person}{Tianbing Xu}, \bibinfo{person}{Bo
  Liu}, \bibinfo{person}{Tao Xu}, \bibinfo{person}{Yanxin Shi},
  \bibinfo{person}{Antoine Atallah}, \bibinfo{person}{Ralf Herbrich},
  \bibinfo{person}{Stuart Bowers}, {et~al\mbox{.}}}
  \bibinfo{year}{2014}\natexlab{}.
\newblock \showarticletitle{Practical lessons from predicting clicks on ads at
  facebook}. In \bibinfo{booktitle}{\emph{Proceedings of the Eighth
  International Workshop on Data Mining for Online Advertising}}. ACM,
  \bibinfo{pages}{1--9}.
\newblock


\bibitem[\protect\citeauthoryear{Kallus and Zhou}{Kallus and Zhou}{2019}]%
        {kallus2019fairness}
\bibfield{author}{\bibinfo{person}{Nathan Kallus} {and} \bibinfo{person}{Angela
  Zhou}.} \bibinfo{year}{2019}\natexlab{}.
\newblock \showarticletitle{The Fairness of Risk Scores Beyond Classification:
  Bipartite Ranking and the xAUC Metric}.
\newblock \bibinfo{journal}{\emph{arXiv:1902.05826}} (\bibinfo{year}{2019}).
\newblock


\bibitem[\protect\citeauthoryear{Kim, Ghorbani, and Zou}{Kim
  et~al\mbox{.}}{2019}]%
        {post19}
\bibfield{author}{\bibinfo{person}{Michael~P. Kim}, \bibinfo{person}{Amirata
  Ghorbani}, {and} \bibinfo{person}{James Zou}.}
  \bibinfo{year}{2019}\natexlab{}.
\newblock \showarticletitle{Multiaccuracy: Black-Box Post-Processing for
  Fairness in Classification}. In \bibinfo{booktitle}{\emph{AIES '19}}.
\newblock


\bibitem[\protect\citeauthoryear{Kumar and Shah}{Kumar and Shah}{2018}]%
        {kumar2018false}
\bibfield{author}{\bibinfo{person}{Srijan Kumar} {and} \bibinfo{person}{Neil
  Shah}.} \bibinfo{year}{2018}\natexlab{}.
\newblock \showarticletitle{False information on web and social media: A
  survey}.
\newblock \bibinfo{journal}{\emph{arXiv preprint arXiv:1804.08559}}
  (\bibinfo{year}{2018}).
\newblock


\bibitem[\protect\citeauthoryear{Louizos, Swersky, Li, Welling, and
  Zemel}{Louizos et~al\mbox{.}}{2016}]%
        {variation16}
\bibfield{author}{\bibinfo{person}{Christos Louizos}, \bibinfo{person}{Kevin
  Swersky}, \bibinfo{person}{Yujia Li}, \bibinfo{person}{Max Welling}, {and}
  \bibinfo{person}{Richard Zemel}.} \bibinfo{year}{2016}\natexlab{}.
\newblock \showarticletitle{The Variational Fair Autoencoder}. In
  \bibinfo{booktitle}{\emph{ICLR}}.
\newblock


\bibitem[\protect\citeauthoryear{Ma, Zhao, Huang, Zhi, Hu, Zhu, and Gai}{Ma
  et~al\mbox{.}}{2018}]%
        {xiao2018}
\bibfield{author}{\bibinfo{person}{Xiao Ma}, \bibinfo{person}{Liqin Zhao},
  \bibinfo{person}{Guan Huang}, \bibinfo{person}{Wang Zhi},
  \bibinfo{person}{Zelin Hu}, \bibinfo{person}{Xiaoqiang Zhu}, {and}
  \bibinfo{person}{Kun Gai}.} \bibinfo{year}{2018}\natexlab{}.
\newblock \showarticletitle{Entire Space Multi-Task Model: An Effective
  Approach for Estimating Post-Click Conversion Rate}. In
  \bibinfo{booktitle}{\emph{SIGIR}}.
\newblock


\bibitem[\protect\citeauthoryear{Madras, Creager, Pitassi, and Zemel}{Madras
  et~al\mbox{.}}{2018}]%
        {DBLP:conf/icml/MadrasCPZ18}
\bibfield{author}{\bibinfo{person}{David Madras}, \bibinfo{person}{Elliot
  Creager}, \bibinfo{person}{Toniann Pitassi}, {and}
  \bibinfo{person}{Richard~S. Zemel}.} \bibinfo{year}{2018}\natexlab{}.
\newblock \showarticletitle{Learning Adversarially Fair and Transferable
  Representations}. In \bibinfo{booktitle}{\emph{ICML}}.
\newblock


\bibitem[\protect\citeauthoryear{Mann and Whitney}{Mann and Whitney}{1947}]%
        {mwu_test}
\bibfield{author}{\bibinfo{person}{H.~B. Mann} {and} \bibinfo{person}{D.~R.
  Whitney}.} \bibinfo{year}{1947}\natexlab{}.
\newblock \showarticletitle{On a Test of Whether one of Two Random Variables is
  Stochastically Larger than the Other}. In \bibinfo{booktitle}{\emph{Annals of
  Mathematical Statistics}}.
\newblock


\bibitem[\protect\citeauthoryear{Mroueh, Sercu, and Goel}{Mroueh
  et~al\mbox{.}}{2017}]%
        {mcgan17}
\bibfield{author}{\bibinfo{person}{Youssef Mroueh}, \bibinfo{person}{Tom
  Sercu}, {and} \bibinfo{person}{Vaibhava Goel}.}
  \bibinfo{year}{2017}\natexlab{}.
\newblock \showarticletitle{McGan: Mean and Covariance Feature Matching GAN}.
  In \bibinfo{booktitle}{\emph{ICML}}.
\newblock


\bibitem[\protect\citeauthoryear{Muandet, Fukumizu, Sriperumbudur, and
  Sch{\"o}lkopf}{Muandet et~al\mbox{.}}{2017}]%
        {MuaFukSriSch17}
\bibfield{author}{\bibinfo{person}{K. Muandet}, \bibinfo{person}{K. Fukumizu},
  \bibinfo{person}{B. Sriperumbudur}, {and} \bibinfo{person}{B.
  Sch{\"o}lkopf}.} \bibinfo{year}{2017}\natexlab{}.
\newblock \showarticletitle{Kernel Mean Embedding of Distributions: A Review
  and Beyond}.
\newblock \bibinfo{journal}{\emph{Foundations and Trends in Machine Learning}}
  \bibinfo{volume}{10}, \bibinfo{number}{1-2} (\bibinfo{year}{2017}),
  \bibinfo{pages}{1--141}.
\newblock


\bibitem[\protect\citeauthoryear{Narasimhan, Cotter, Gupta, and
  Wang}{Narasimhan et~al\mbox{.}}{2020}]%
        {narasimhan2019pairwise}
\bibfield{author}{\bibinfo{person}{Harikrishna Narasimhan},
  \bibinfo{person}{Andrew Cotter}, \bibinfo{person}{Maya~R. Gupta}, {and}
  \bibinfo{person}{Serena Wang}.} \bibinfo{year}{2020}\natexlab{}.
\newblock \showarticletitle{Pairwise Fairness for Ranking and Regression}.
\newblock  (\bibinfo{year}{2020}), \bibinfo{pages}{5248--5255}.
\newblock


\bibitem[\protect\citeauthoryear{Pleiss, Raghavan, Wu, Kleinberg, and
  Weinberger}{Pleiss et~al\mbox{.}}{2017}]%
        {calibration17}
\bibfield{author}{\bibinfo{person}{Geoff Pleiss}, \bibinfo{person}{Manish
  Raghavan}, \bibinfo{person}{Felix Wu}, \bibinfo{person}{Jon Kleinberg}, {and}
  \bibinfo{person}{Kilian~Q. Weinberger}.} \bibinfo{year}{2017}\natexlab{}.
\newblock \showarticletitle{On Fairness and Calibration}. In
  \bibinfo{booktitle}{\emph{NIPS 2017}}.
\newblock


\bibitem[\protect\citeauthoryear{Potthast, K{\"o}psel, Stein, and
  Hagen}{Potthast et~al\mbox{.}}{2016}]%
        {potthast2016clickbait}
\bibfield{author}{\bibinfo{person}{Martin Potthast}, \bibinfo{person}{Sebastian
  K{\"o}psel}, \bibinfo{person}{Benno Stein}, {and} \bibinfo{person}{Matthias
  Hagen}.} \bibinfo{year}{2016}\natexlab{}.
\newblock \showarticletitle{Clickbait detection}. In
  \bibinfo{booktitle}{\emph{European Conference on Information Retrieval}}.
  Springer, \bibinfo{pages}{810--817}.
\newblock


\bibitem[\protect\citeauthoryear{Prost, Qian, Chen, Chi, Chen, and
  Beutel}{Prost et~al\mbox{.}}{2019}]%
        {prost19}
\bibfield{author}{\bibinfo{person}{Flavien Prost}, \bibinfo{person}{Hai Qian},
  \bibinfo{person}{Qiuwen Chen}, \bibinfo{person}{Ed~H. Chi},
  \bibinfo{person}{Jilin Chen}, {and} \bibinfo{person}{Alex Beutel}.}
  \bibinfo{year}{2019}\natexlab{}.
\newblock \showarticletitle{Toward a better trade-off between performance and
  fairness with kernel-based distribution matching}.
\newblock \bibinfo{journal}{\emph{arXiv preprint arXiv:1910.11779}}
  (\bibinfo{year}{2019}).
\newblock


\bibitem[\protect\citeauthoryear{Radlinski, Kleinberg, and Joachims}{Radlinski
  et~al\mbox{.}}{2008}]%
        {diverse_icml08}
\bibfield{author}{\bibinfo{person}{Filip Radlinski}, \bibinfo{person}{Robert
  Kleinberg}, {and} \bibinfo{person}{Thorsten Joachims}.}
  \bibinfo{year}{2008}\natexlab{}.
\newblock \showarticletitle{Learning Diverse Rankings with Multi-Armed
  Bandits}. In \bibinfo{booktitle}{\emph{ICML}}.
\newblock


\bibitem[\protect\citeauthoryear{Singh and Joachims}{Singh and
  Joachims}{2018}]%
        {DBLP:conf/kdd/SinghJ18}
\bibfield{author}{\bibinfo{person}{Ashudeep Singh} {and}
  \bibinfo{person}{Thorsten Joachims}.} \bibinfo{year}{2018}\natexlab{}.
\newblock \showarticletitle{Fairness of Exposure in Rankings}. In
  \bibinfo{booktitle}{\emph{{KDD}' 18}}. \bibinfo{pages}{2219--2228}.
\newblock


\bibitem[\protect\citeauthoryear{Singh and Joachims}{Singh and
  Joachims}{2019}]%
        {singh2019policy}
\bibfield{author}{\bibinfo{person}{Ashudeep Singh} {and}
  \bibinfo{person}{Thorsten Joachims}.} \bibinfo{year}{2019}\natexlab{}.
\newblock \showarticletitle{Policy learning for fairness in ranking}. In
  \bibinfo{booktitle}{\emph{Advances in Neural Information Processing
  Systems}}. \bibinfo{pages}{5426--5436}.
\newblock


\bibitem[\protect\citeauthoryear{Slivkins, Radlinski, and Gollapudi}{Slivkins
  et~al\mbox{.}}{2010}]%
        {diverse_icml10}
\bibfield{author}{\bibinfo{person}{Aleksandrs Slivkins}, \bibinfo{person}{Filip
  Radlinski}, {and} \bibinfo{person}{Sreenivas Gollapudi}.}
  \bibinfo{year}{2010}\natexlab{}.
\newblock \showarticletitle{Learning optimally diverse rankings over large
  document collections}. In \bibinfo{booktitle}{\emph{ICML}}.
\newblock


\bibitem[\protect\citeauthoryear{\v{Z}liobait\.{e}}{\v{Z}liobait\.{e}}{2015}]%
        {liobait2015OnTR}
\bibfield{author}{\bibinfo{person}{Indr\.{e} \v{Z}liobait\.{e}}.}
  \bibinfo{year}{2015}\natexlab{}.
\newblock \showarticletitle{On the relation between accuracy and fairness in
  binary classification}.
\newblock \bibinfo{journal}{\emph{ArXiv}}  \bibinfo{volume}{abs/1505.05723}
  (\bibinfo{year}{2015}).
\newblock


\bibitem[\protect\citeauthoryear{Wang, Lin, and Metzler}{Wang
  et~al\mbox{.}}{2011}]%
        {wang2011cascade}
\bibfield{author}{\bibinfo{person}{Lidan Wang}, \bibinfo{person}{Jimmy Lin},
  {and} \bibinfo{person}{Donald Metzler}.} \bibinfo{year}{2011}\natexlab{}.
\newblock \showarticletitle{A cascade ranking model for efficient ranked
  retrieval}. In \bibinfo{booktitle}{\emph{Proceedings of the 34th
  international ACM SIGIR conference on Research and development in Information
  Retrieval}}. ACM, \bibinfo{pages}{105--114}.
\newblock


\bibitem[\protect\citeauthoryear{Yao and Huang}{Yao and Huang}{2017}]%
        {yao2017beyond}
\bibfield{author}{\bibinfo{person}{Sirui Yao} {and} \bibinfo{person}{Bert
  Huang}.} \bibinfo{year}{2017}\natexlab{}.
\newblock \showarticletitle{Beyond parity: Fairness objectives for
  collaborative filtering}. In \bibinfo{booktitle}{\emph{Advances in Neural
  Information Processing Systems}}. \bibinfo{pages}{2921--2930}.
\newblock


\bibitem[\protect\citeauthoryear{Yi, Hong, Zhong, Liu, and Rajan}{Yi
  et~al\mbox{.}}{2014}]%
        {yi2014beyond}
\bibfield{author}{\bibinfo{person}{Xing Yi}, \bibinfo{person}{Liangjie Hong},
  \bibinfo{person}{Erheng Zhong}, \bibinfo{person}{Nanthan~Nan Liu}, {and}
  \bibinfo{person}{Suju Rajan}.} \bibinfo{year}{2014}\natexlab{}.
\newblock \showarticletitle{Beyond clicks: dwell time for personalization}. In
  \bibinfo{booktitle}{\emph{Proceedings of the 8th ACM Conference on
  Recommender systems}}. ACM, \bibinfo{pages}{113--120}.
\newblock


\bibitem[\protect\citeauthoryear{Zafar, Valera, Rodriguez, and Gummadi}{Zafar
  et~al\mbox{.}}{2015}]%
        {Zafar2015LearningFC}
\bibfield{author}{\bibinfo{person}{Muhammad~Bilal Zafar},
  \bibinfo{person}{Isabel Valera}, \bibinfo{person}{Manuel~Gomez Rodriguez},
  {and} \bibinfo{person}{Krishna~P. Gummadi}.} \bibinfo{year}{2015}\natexlab{}.
\newblock \showarticletitle{Learning Fair Classifiers}.
\newblock


\bibitem[\protect\citeauthoryear{Zafar, Valera, Rogriguez, and Gummadi}{Zafar
  et~al\mbox{.}}{2017}]%
        {pmlr-v54-zafar17a}
\bibfield{author}{\bibinfo{person}{Muhammad~Bilal Zafar},
  \bibinfo{person}{Isabel Valera}, \bibinfo{person}{Manuel~Gomez Rogriguez},
  {and} \bibinfo{person}{Krishna~P. Gummadi}.} \bibinfo{year}{2017}\natexlab{}.
\newblock \showarticletitle{{Fairness Constraints: Mechanisms for Fair
  Classification}}. In \bibinfo{booktitle}{\emph{Proceedings of the 20th
  International Conference on Artificial Intelligence and Statistics}}.
  \bibinfo{publisher}{PMLR}, \bibinfo{address}{Fort Lauderdale, FL, USA}.
\newblock


\bibitem[\protect\citeauthoryear{Zehlike, Bonchi, Castillo, Hajian, Megahed,
  and Baeza-Yates}{Zehlike et~al\mbox{.}}{2017}]%
        {zehlike2017fa}
\bibfield{author}{\bibinfo{person}{Meike Zehlike}, \bibinfo{person}{Francesco
  Bonchi}, \bibinfo{person}{Carlos Castillo}, \bibinfo{person}{Sara Hajian},
  \bibinfo{person}{Mohamed Megahed}, {and} \bibinfo{person}{Ricardo
  Baeza-Yates}.} \bibinfo{year}{2017}\natexlab{}.
\newblock \showarticletitle{Fa* ir: A fair top-k ranking algorithm}. In
  \bibinfo{booktitle}{\emph{CIKM'17}}. \bibinfo{pages}{1569--1578}.
\newblock


\bibitem[\protect\citeauthoryear{Zhang, Lemoine, and Mitchell}{Zhang
  et~al\mbox{.}}{2018}]%
        {adv_learning18}
\bibfield{author}{\bibinfo{person}{Brian~Hu Zhang}, \bibinfo{person}{Blake
  Lemoine}, {and} \bibinfo{person}{Margaret Mitchell}.}
  \bibinfo{year}{2018}\natexlab{}.
\newblock \showarticletitle{Mitigating Unwanted Biases with Adversarial
  Learning}. In \bibinfo{booktitle}{\emph{Association for the Advancement of
  Artificial Intelligence}}.
\newblock


\bibitem[\protect\citeauthoryear{Zhao, Hong, Wei, Chen, Nath, Andrews,
  Kumthekar, Sathiamoorthy, Yi, and Chi}{Zhao et~al\mbox{.}}{2019}]%
        {zhe2019}
\bibfield{author}{\bibinfo{person}{Zhe Zhao}, \bibinfo{person}{Lichan Hong},
  \bibinfo{person}{Li Wei}, \bibinfo{person}{Jilin Chen},
  \bibinfo{person}{Aniruddh Nath}, \bibinfo{person}{Shawn Andrews},
  \bibinfo{person}{Aditee Kumthekar}, \bibinfo{person}{Maheswaran
  Sathiamoorthy}, \bibinfo{person}{Xinyang Yi}, {and} \bibinfo{person}{Ed
  Chi}.} \bibinfo{year}{2019}\natexlab{}.
\newblock \showarticletitle{Recommending What Video to Watch next: A Multitask
  Ranking System}. In \bibinfo{booktitle}{\emph{Proceedings of the 13th ACM
  Conference on Recommender Systems}} \emph{(\bibinfo{series}{RecSys ’19})}.
\newblock


\end{thebibliography}

\end{document}